\documentclass[11pt]{article}

\usepackage[final]{acl}

\usepackage{times}
\usepackage{latexsym}
\usepackage{amsmath}
\usepackage{amssymb} 
\usepackage{algorithm}
\usepackage{algpseudocode}
\usepackage{amsthm}
\newtheorem{theorem}{Theorem}[section]
\newtheorem{proposition}[theorem]{Proposition}
\newtheorem{observation}[theorem]{observation}
\usepackage{amsfonts}
\usepackage[T1]{fontenc}
\usepackage[table]{xcolor}

\usepackage{tcolorbox}
\usepackage{setspace}
\usepackage{booktabs}

\usepackage[utf8]{inputenc}

\usepackage{microtype}

\usepackage{inconsolata}

\definecolor{color1}{RGB}{239, 118, 84} 
\definecolor{color2}{RGB}{255, 215, 0} 
\newcommand{\fasymbol}{\textcolor{color1}{$\blacklozenge$}}
\newcommand{\tx}{\textcolor{color2}{$\blacklozenge$}}

\newcommand{\factemoji}{\raisebox{-0.20\height}{\includegraphics[height=1.6\fontcharht\font`\B]{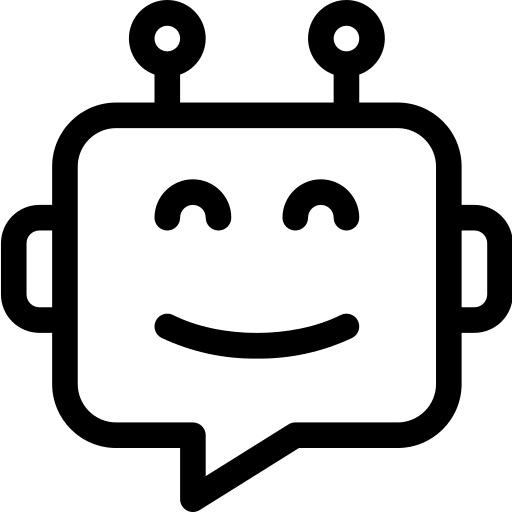}}}
\title{\factemoji{} I2E: From Image Pixels to Actionable Interactive Environments for Text-Guided Image Editing}

\makeatletter
\def\thanks#1{\protected@xdef\@thanks{\@thanks
        \protect\footnotetext{#1}}}
\makeatother

\usepackage{graphicx}
\usepackage{fontawesome5}
\algnewcommand\Input{\State \textbf{Input:}}
\algnewcommand\Output{\State \textbf{Output:}}

\title{\factemoji{} I2E: From Image Pixels to Actionable Interactive Environments for Text-Guided Image Editing}

\author{
 \textbf{Jinghan Yu\textsuperscript{\fasymbol{}1,2}},
 \textbf{Junhao Xiao\thanks{\fasymbol{}: Equal contribution as co-first authors. Work completed during joint internship at HUST and Kuaishou.}\textsuperscript{\fasymbol{}1,2,3}},
 \textbf{Chenyu Zhu\textsuperscript{1,2}},
 \textbf{Jiaming Li\textsuperscript{1,2}},
 \textbf{Jia Li\textsuperscript{1}},
 \textbf{Hanming Deng\textsuperscript{1}},\\
 \textbf{Xirui Wang\textsuperscript{1}},
 \textbf{Guoli Jia\textsuperscript{4}},
 \textbf{Jianjun Li\textsuperscript{1}},
 \textbf{Xiang Bai\textsuperscript{1}},
 \textbf{Bowen Zhou\textsuperscript{4,5}},
 \textbf{Zhiyuan Ma\thanks{\tx{}: Corresponding author: Zhiyuan Ma.}\textsuperscript{\tx{}1}}
\\
 \textsuperscript{1}Huazhong University of Science and Technology,
 \textsuperscript{2}Kuaishou Technology,
 \\
 \textsuperscript{3}Central China Normal University,
 \textsuperscript{4}Tsinghua University,
 \textsuperscript{5}Shanghai AI Laboratory
\\
 {\small \texttt{jinghanyu0917@gmail.com, xiaojunhao066@gmail.com, mzyth@hust.edu.cn}}
}

\begin{document}
\maketitle

\begin{abstract}

Existing text-guided image editing methods primarily rely on end-to-end pixel-level inpainting paradigm. 
Despite its success in simple scenarios, this paradigm still significantly struggles with compositional editing tasks that require precise local control and complex multi-object spatial reasoning. This paradigm is severely limited by \textbf{\emph{1) the implicit coupling of planning and execution}}, \textbf{\emph{2) the lack of object-level control granularity}}, and \textbf{\emph{3) the reliance on unstructured, pixel-centric modeling}}. To address these limitations, we propose I2E, a novel ``\emph{Decompose-then-Action}'' paradigm that revisits image editing as an actionable interaction process within a structured environment. I2E utilizes a Decomposer to transform unstructured images into discrete, manipulable object layers and then introduces a physics-aware Vision-Language-Action Agent to parse complex instructions into a series of atomic actions via Chain-of-Thought reasoning. Further, we also construct \textsc{I2E-Bench}, a benchmark designed for multi-instance spatial reasoning and high-precision editing. Experimental results on \textsc{I2E-Bench} and multiple public benchmarks demonstrate that I2E significantly outperforms state-of-the-art methods in handling complex compositional instructions, maintaining physical plausibility, and ensuring multi-turn editing stability. Code and dataset: \href{https://image2env.github.io/}{\texttt{{project page}}}.
\end{abstract}

\begin{figure}[t]
  \centering
  \includegraphics[width=0.9\linewidth]{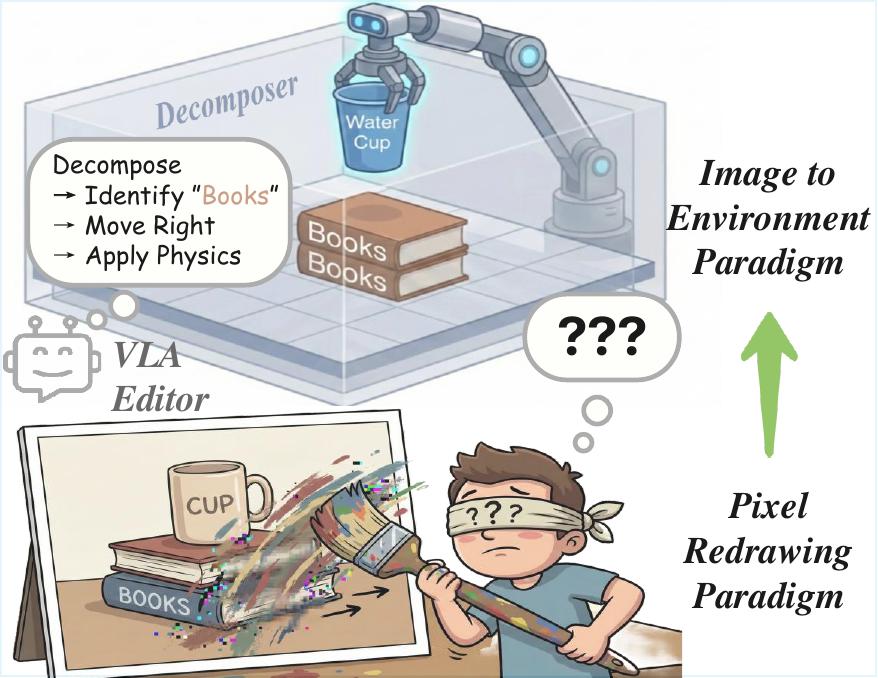}
  \caption{\textbf{Paradigm Comparison}. Unlike the \emph{Pixel Redrawing Paradigm} that directly manipulates pixels, I2E transforms images into a structured environment, enabling the VLA Editor to perform spatial and physical reasoning for precise, physically plausible edits.}
  \vspace{-15pt}
  \label{fig:abs}
\end{figure}

\section{Introduction}

When performing image editing, humans rarely reason in terms of direct pixel manipulations. Consider a typical editing request: \emph{``Move the books that are pressed by the water cup on the desktop to the right side of the cup''}. For humans, such an instruction implicitly involves a sequence of intermediate reasoning steps, including object identification, spatial relationship understanding, physical constraint awareness, and ordered execution.

\begin{figure*}[t]
  \centering
  \includegraphics[width=\textwidth]{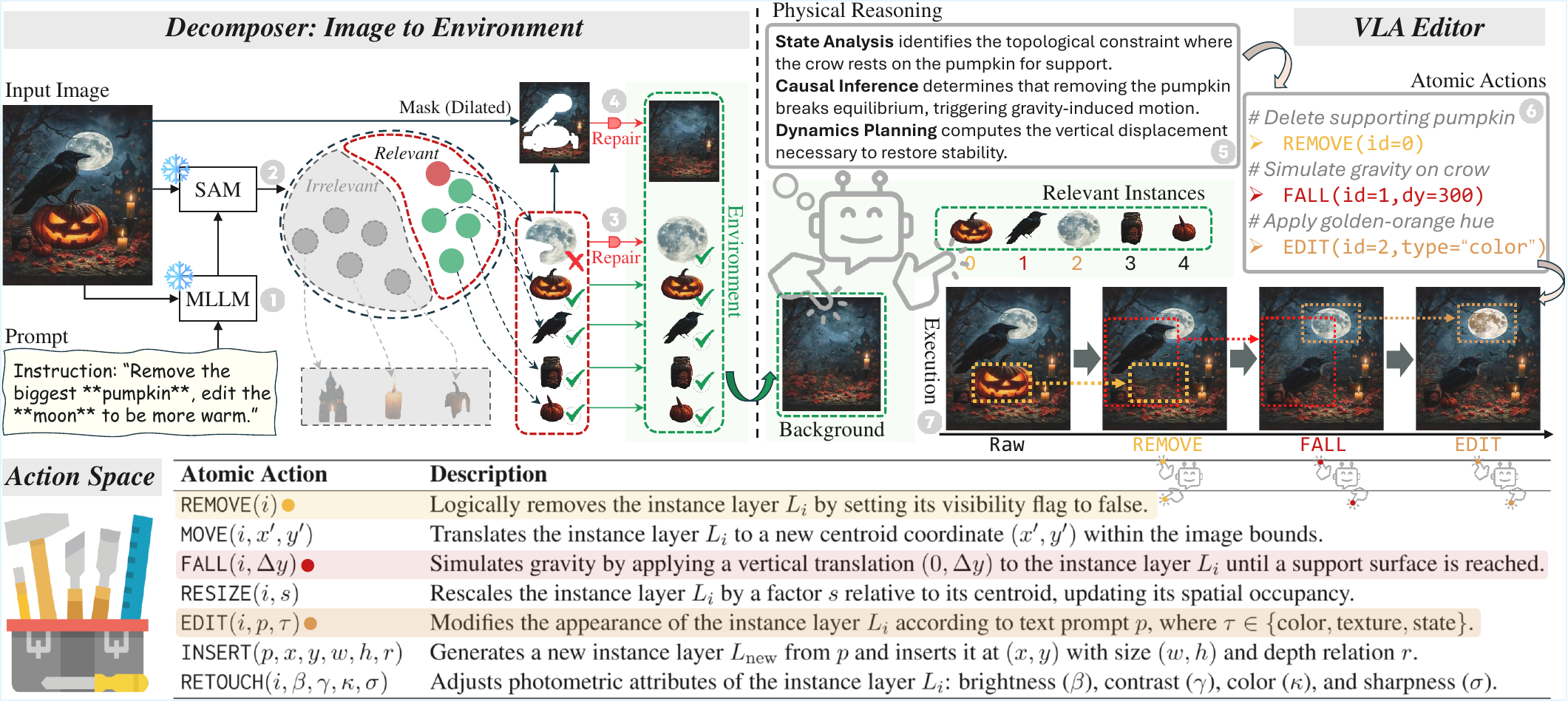}
  \caption{Overview of the I2E. The \textit{Decomposer} transforms unstructured images into a structured environment of actionable physical layers. The physics-aware \textit{VLA Editor} then uses chain-of-thought reasoning to translate instructions into executable atomic actions (see bottom) and executes them sequentially.}
  \label{fig:pipeline}
  \vspace{-10pt}
\end{figure*}

However, existing text-guided image editing models generally exploit an \emph{End-to-End} paradigm but lack such intermediate representations and \emph{Reasoning-then-Action} process. Specifically, they typically attempt to directly map instructions to final results through one or multiple rounds of pixel-level redrawing. While this end-to-end pixel redrawing paradigm is generally effective in simple editing scenarios, it exposes three major structural limitations when applied to compositional editing tasks that require precise local control and complex multi-object spatial reasoning:

(i) \textbf{\textit{Tight coupling between semantic reasoning and execution.}} Models are required to perform instruction understanding and pixel synthesis within a single generation process, making it difficult to form stable intermediate decision structures and leading to significantly degraded instruction-following performance in complex scenarios.

(ii) \textbf{\textit{Lack of object-level representations and boundaries.}} When editing is performed directly in pixel space, modifications cannot be strictly confined to target instances and often propagate as global perturbations to non-target regions.

(iii) \textbf{\textit{Pixel-centric and unstructured modeling.}} By treating images as unstructured two-dimensional pixel collections, models struggle to explicitly represent depth relations, support relations, and scale constraints, which frequently results in physically implausible editing outcomes, such as \emph{``floating objects''} (As described in Figure~\ref{fig:qualitative}).

These issues are further amplified in multi-round incremental interactive editing. Since each editing step typically redraws the entire image based on the previous output, unstructured pixel-level updates cause errors to accumulate across iterations, leading to severe \emph{``feature drift''} and making it difficult to achieve fine-grained, controllable continuous editing (see Figure~\ref{fig:delta}). Moreover, repeatedly invoking computationally expensive generation processes substantially degrades interaction efficiency.

To address these challenges, as illustrated in Figure~\ref{fig:abs}, we propose a new image editing paradigm, I2E (Image-to-Environment), which \textbf{reformulates image editing as an interactive process within an actionable structured environment}. From this perspective, an image is no longer treated as an indivisible pixel array, but as a composition of entities and background with explicit spatial relationships. Building on this, I2E mainly operates in two stages (Figure~\ref{fig:pipeline}): \emph{1) Image-to-Environment Transition.} A \textit{Decomposer} module transforms unstructured pixel representations into environment representations with explicit spatial structure. This module explicitly recovers the complete appearance of each instance object (\emph{e.g., the obscured moon in Figure~\ref{fig:pipeline}}) and their relative physical relationships, encapsulating them as independent and manipulable object-level physical layers, which together with the background layer form an interactive actionable environment. \emph{2) VLA-based Environment Editing. } On top of this structured environment, we introduce a physics-aware \textit{Vision--Language--Action (VLA) Editor} as the core decision-making component. Rather than directly predicting pixel-level changes, the agent progressively decomposes complex natural language instructions into a sequence of precise atomic actions that satisfy physical constraints through chain-of-thought reasoning.

This design brings multiple advantages. By decomposing high-level editing intents into executable atomic steps, it substantially improves instruction-following for complex instructions. Action execution grounded in object-level environment states further ensures that edits are strictly localized to target instances, effectively eliminating interference in irrelevant regions. Beyond these benefits, the decoupled atomic actions and the explicit grounding of target entities enable efficient multi-round incremental editing. This transforms the generative editing paradigm from a ``one-shot global repainting'' process into a ``progressive refinement workflow''. When responding to user feedback or during self-correction, the system does not require a reset of the scene, but instead updates the state by appending corrective actions.

Moreover, we observe that existing benchmarks primarily focus on style transfer or simple single-step instructions, lacking comprehensive evaluation of complex spatial reasoning, multi-instance interaction, and physical constraint consistency. To fill this gap and validate the effectiveness of the proposed approach, we introduce \textsc{I2E-Bench}, a benchmark designed for multi-instance spatial reasoning and high-precision image editing. Extensive experiments on \textsc{I2E-Bench}, as well as public benchmarks such as MagicBrush and EmuEdit, demonstrate that I2E significantly outperforms state-of-the-art methods in handling compositional instructions, maintaining physical consistency, and ensuring stability in multi-round interactions.

\section{Related Works}

\subsection{Text-Guided Image Editing and Agentic Approaches}

Early text-guided image editing methods \citep{brooks2023instructpix2pix, hertz2022prompt, meng2021sdedit} primarily rely on end-to-end pixel redrawing, directly mapping textual instructions to global image synthesis. While effective for simple edits, tightly coupling instruction understanding with pixel generation limits their ability to handle compositional commands that require precise local control and multi-object spatial reasoning.

Recent approaches improve semantic interpretation by incorporating multimodal large language models (MLLMs) \citep{fu2024mgie, liu2025step1x, yu2025anyedit} or unifying reasoning and generation within large transformer-based models \citep{xiao2025omnigen, betker2023improving,feng2025dit4edit}. Despite stronger instruction comprehension, edit execution remains bound to global resampling, making these methods prone to unintended changes in non-target regions and shows attribute leakage\citep{mun2025addressing}.

Agentic editing frameworks \citep{huang2024smartedit, hu2025image} further decompose instructions into sub-tasks, yet most still realize each step through independent image generation, leading to accumulated deviations across interactions. In contrast, I2E performs editing as object-level actions over a structured environment, enabling incremental updates without re-synthesizing the entire image.

\subsection{Structured Scene Representation and Amodal Decomposition}

Precise local editing requires structured scene representations beyond flat pixel grids. Instance segmentation models \citep{kirillov2023segment, ravi2024sam} provide object localization but operate at the modal level, failing to recover occluded content and often introducing missing regions when objects are manipulated \citep{yu2019free}.

Layered image generation methods \citep{zhang2024transparent} partially improve locality but typically focus on synthesis with fixed layouts rather than interactive editing. Amodal completion approaches \citep{ozguroglu2024pix2gestalt, liu2025towards} reconstruct occluded appearances, yet are usually designed as standalone restoration modules and lack a unified representation for downstream manipulation \citep{ao2025open}.

Our work integrates instance-level amodal decomposition with explicit depth-aware ordering to form complete, spatially organized object layers. This representation enables object-specific editing while keeping non-target regions unchanged.

\subsection{Physical Reasoning and Vision--Language--Action Models}

Maintaining physical plausibility remains challenging in generative image editing. Existing methods often rely on static geometric cues, such as depth maps or edge constraints \citep{zhang2023adding, lee2022instaorder}, which cannot capture changes in physical relationships induced by editing actions, leading to implausible outcomes such as unsupported objects \citep{pan2023drag}.

Vision--Language--Action (VLA) models \citep{zitkovich2023rt, driess2023palm, kim2024openvla} demonstrate effective physical reasoning by grounding language instructions in structured environment states and executing actions under explicit constraints. Inspired by this paradigm, we reformulate image editing as task-driven interaction within a structured scene representation \citep{ha2018world}. Rather than directly synthesizing pixels, our framework plans and executes object-level actions conditioned on explicit spatial and relational states, enabling physically consistent editing.

\section{Motivation: Analysis of End-to-End Editing Bottlenecks}

End-to-end pixel inpainting, while effective for simple edits, struggles with compositional tasks that require precise local control and multi-instance spatial reasoning. We analyze these structural bottlenecks to motivate a paradigm shift.

\subsection{Instruction Collapse}

Figure~\ref{fig:qualitative} illustrates that when instructions contain multiple sub-goals, existing models often execute only a subset, ignoring or inconsistently satisfying others (\emph{instruction collapse}). Analysis (Appendix~\ref{app:tho1}) attributes this to text encoding and conditioning limits: complex instructions are compressed into a single global embedding and injected via cross-attention, causing sub-goal conflicts and unstable execution.

\begin{figure}[t]
  \centering
  \includegraphics[width=\linewidth]{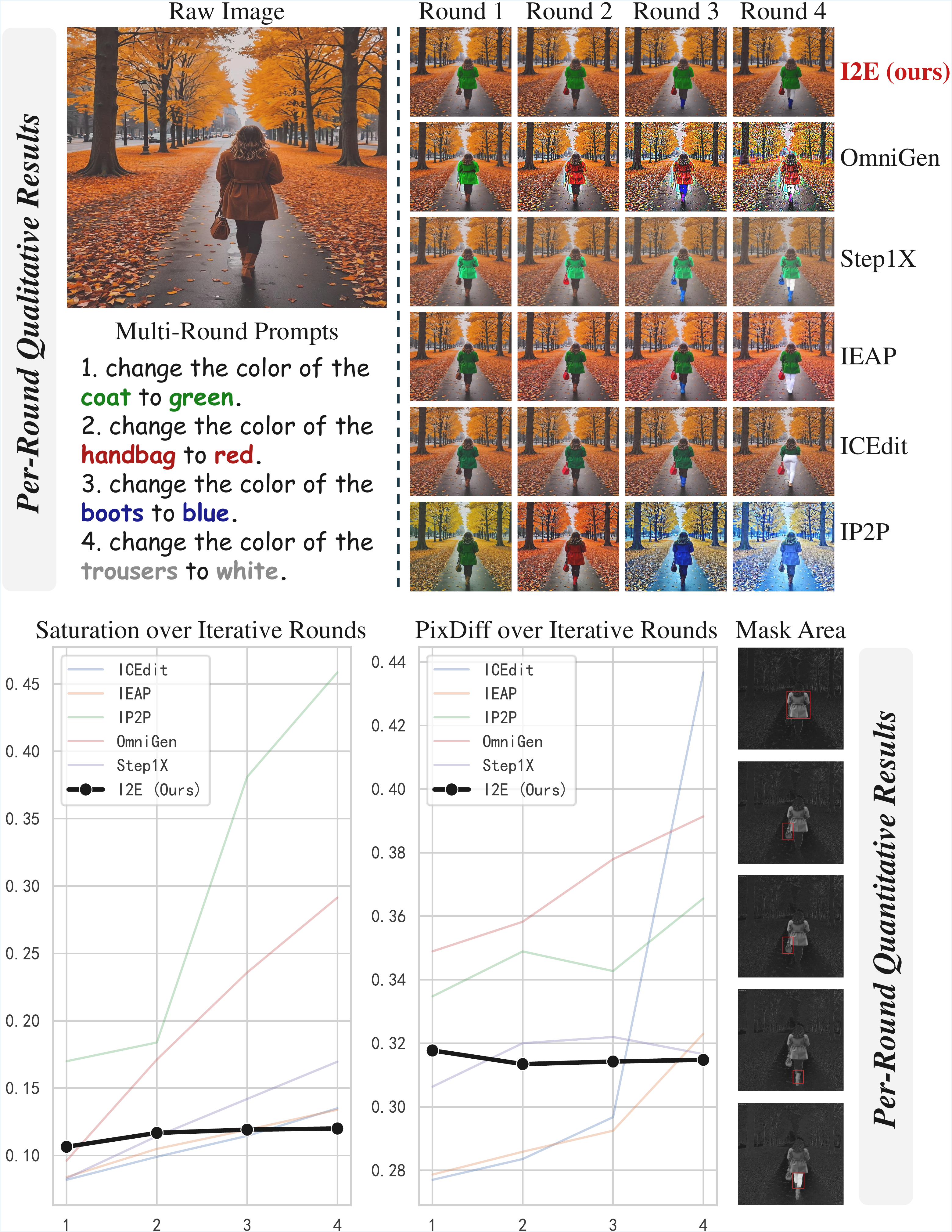}
  \caption{Multi-turn Stability.\textbf{Top:} Baselines exhibit severe \textit{error accumulation} (e.g., visual distortion) over 4 rounds, while I2E preserves integrity. \textbf{Bottom:} Metrics confirm I2E's constant consistency in Saturation and Pixel Difference (PixDiff) versus the monotonic degradation of end-to-end models.}
  \label{fig:delta}
  \vspace{-15pt}
\end{figure}

\subsection{Inevitability of Global Entanglement}

End-to-end models are statistical generators, not deterministic editors, making \emph{lossless local editing} theoretically infeasible. Two structural factors contribute: (i) \textit{VAE bottleneck}: lossy compression-reconstruction ($x \to z \to \hat{x}$) degrades high-frequency details in non-edited regions; (ii) \textit{Self-attention coupling}: local feature changes propagate globally via dense attention. Multi-round simulations (Figure~\ref{fig:delta}) confirm that non-target deviations accumulate with editing steps $n$, leading to severe visual artifacts.

\subsection{Paradigm Shift: From Pixel Resampling to Entity Manipulation}

These observations suggest that tightly coupling instruction parsing, planning, and pixel rendering limits edit reliability. 
We propose a paradigm shift: reformulate image editing by transforming unstructured pixel arrays into interactive structured environments, where a VLA agent executes edits through explicit entity manipulations instead of global pixel resampling.

\section{Methodology}

As illustrated in Figure~\ref{fig:pipeline}, we propose the I2E framework, which reformulates image editing as an interaction process within a actionable environment.
The framework consists of two cascaded stages: (i) \textit{Environment Construction}, where a Decomposer transforms the input image into an explicit object-level environment representation; and (ii) \textit{Agentic Interaction}, in which a VLA Editor performs physical reasoning and executes edits through atomic entity-level operations.

\subsection{Decomposer: Environment Construction}

The Decomposer converts an unstructured input image $I$ into an interactive structured environment $\mathcal{E}$ for object-level manipulation.
It disentangles and completes relevant instances as independent entities, and organizes them into a physically consistent stacking hierarchy.
The resulting object layers with explicit spatial relationships form a manipulable environment for subsequent agentic interaction.

\paragraph{Instance Disentanglement and Completion.}
To lift the input image into manipulable layers, we first employ a collaborative perception pipeline to identify and segment high-precision masks $m_i$ for relevant instances (i.e., potential editing targets), while merging irrelevant objects into the background. This stage integrates a Multimodal Large Language Model (MLLM) for semantic reasoning with advanced grounding and segmentation frameworks to ensure mask accuracy. 
Since the segmented regions are inherently incomplete due to occlusion, we utilize a generative fill-in mechanism to recover invisible structures. Guided by context-rich prompts from the MLLM, this process synthesizes missing textures and geometry, yielding a set of complete, transparent RGBA layers $\{\tilde{I}_i\}$. Concurrent with foreground processing, an occlusion-aware inpainting module is applied to remove the extracted instances from the original canvas, restoring a clean and cohesive background $B$. See details in Appendix~\ref{app:implementation}

\paragraph{Physical Layer Construction.}
Establishing a globally correct stacking order is a prerequisite for a physically consistent environment.
Since explicit occlusion constraints only exist in regions where instances overlap at the pixel level, we propose a DAG-based Spatial Constraint Propagation Algorithm (see Algorithm~\ref{alg:spatial_ordering}) to infer global layer relationships. Specifically, we construct a directed acyclic graph (DAG) where nodes correspond to instances and directed edges encode occlusion dependencies.
We jointly consider two types of constraints:
(i) \emph{hard constraints} derived from the pixel-level occlusion matrix predictions, and
(ii) \emph{soft constraints} obtained from monocular depth estimation, which refine the relative ordering without violating observed hard occlusions.
By computing the transitive closure of the graph and defining the node out-degree as the depth score $D_i$, we resolve the global topological structure.

The global stacking sequence is then formalized as a permutation $\pi = (\pi_1, \ldots, \pi_N)$ that satisfies the monotonicity constraint:
\begin{equation}
D_{\pi_k} \ge D_{\pi_{k+1}}, \quad \forall k \in [1, N-1],
\end{equation}
where $\pi_1$ denotes the front-most instance index.
Finally, each instance is encapsulated into an independent physical layer $L_i$, and combined with the background $B$ to constitute the structured physical environment $\mathcal{E}$:
\begin{equation}
L_i = \{\tilde{I}_i, m_i, D_i\}, \quad \mathcal{E} = (\{L_i\}_{i=1}^N, B).
\end{equation}

\begin{algorithm}[t]
\captionsetup{font={footnotesize}}
\caption{DAG-based Spatial Constraint Propagation}
\label{alg:spatial_ordering}

\scriptsize

\begin{algorithmic}[1]
\renewcommand{\algorithmicrequire}{\textbf{Input:}}
\renewcommand{\algorithmicensure}{\textbf{Output:}}

\Require Occlusion matrix $O \in \{0,1\}^{N \times N}$, Depth soft constraints $O^{\mathrm{soft}} \in \{0,1\}^{N \times N}$
\Ensure Depth scores $D \in \mathbb{Z}^N$
\State \textit{\# Phase 1: Occlusion (hard constraints)}
\ForAll{$i,j \in \{1, \dots, N\}$}
    \If{$O_{ij}=1$}
        \State $G_{ji}\leftarrow 1$
    \EndIf
\EndFor

\State \textit{\# Phase 2: Depth (soft constraints)}
\ForAll{$i,j \in \{1, \dots, N\}$}
    \If{$O^{\mathrm{soft}}_{ij}=1 \land G_{ji}=0$}
        \State $G_{ij}\leftarrow 1$
    \EndIf
\EndFor

\State \textit{\# Phase 3: Constraint propagation}
\Repeat
    \State $G \leftarrow G \lor (G \cdot G)$
\Until{$G$ converges}

\State \textit{\# Phase 4: Calculate Scores}
\For{$i=1$ to $N$}
    \State $D_i \leftarrow \sum_{j \neq i} \mathbb{I}(G_{ij}=1)$
\EndFor

\State \Return $D = \{D_1, \dots, D_N\}$
\end{algorithmic}
\end{algorithm}

\subsection{VLA Editor: Agentic Interaction}

Given the structured environment $\mathcal{E}$, the VLA Editor serves as the decision and execution core, translating natural language instructions into physics-consistent actions that drive environment evolution.

\paragraph{Physics-Aware CoT Reasoning.}
We employ an MLLM-based~\cite{qwen3} agent to perform chain-of-thought (CoT) reasoning under an explicit set of physical constraints $\mathcal{C}_{\mathrm{phy}}$ (e.g., gravity and support rules; see Appendix~\ref{app:physical}) and a predefined action space $\mathcal{A}$ (illustrated in Figure~\ref{fig:pipeline}).
Given the instruction $T$ and the current scene state, the agent produces structured reasoning outputs, which are compiled into a sequence of parameterized atomic actions $\tilde{\mathcal{A}} = \{a_1, \ldots, a_k\}$.

\paragraph{Action Execution.}
Each atomic action updates $\mathcal{E}$ through object-level operations rather than global pixel resampling.
\texttt{REMOVE} hides the target layer, exposing the pre-repaired background $B$.
\texttt{MOVE} and \texttt{FALL} perform rigid transformations on object layers while preserving geometric integrity.
\texttt{RESIZE} rescales object layers with fixed aspect ratios.
Appearance edits are handled via \texttt{EDIT} and \texttt{RETOUCH}, which modify color, texture, or photometric attributes at the layer level.
\texttt{INSERT} synthesizes a new object layer and inserts it into the global stacking order $\pi$ according to predicted relational constraints, ensuring physically consistent occlusion.

\paragraph{Multi-Round Incremental Refinement.}
Since the environment state is explicitly maintained, I2E naturally supports incremental editing through action accumulation.
User feedback or self-critique is handled by appending corrective actions without resetting the scene.
By iteratively executing, evaluating, and revising, the closed loop stabilizes refinement and inhibits the compounding errors that arise with repeated pixel-level re-generation.

\section{I2E-Bench}
Existing benchmarks mainly target style transfer or simple single-step edits, and therefore inadequately evaluate complex spatial reasoning, multi-instance interaction, and physical consistency.
To fill this gap, we introduce \textsc{I2E-Bench}, a benchmark for multi-instance spatial reasoning and high-precision image editing.
It comprises $200$ curated images from diverse open-source platforms, spanning real-world scenes, illustrations, and anime.
Each image is paired with $5$--$10$ editing instructions, emphasizing complex multi-action edits that require precise spatial manipulation while preserving stylistic and semantic coherence.
Details in Appendix~\ref{app:I2Ebench}.

\section{Comparative Experiments}

\textbf{Baselines.}
We compare our method with representative state-of-the-art instruction-guided image editing approaches, including IP2P\citep{brooks2023instructpix2pix}, OmniGen\citep{xiao2025omnigen}, Step1X\citep{liu2025step1x}, IEAP\citep{hu2025image}, and ICEdit\citep{zhang2025context}.
These methods cover the dominant end-to-end and agent-based paradigms in the literature.
For a fair comparison, we restrict all methods to a single refinement round.

\noindent\textbf{Datasets.}
In addition to our proposed \textsc{I2E-Bench}, we evaluate all methods on two widely adopted public benchmarks, MagicBrush\citep{zhang2023magicbrush} and EmuEdit\citep{sheynin2024emu}, to ensure comprehensive and objective evaluation. 

\noindent\textbf{Metrics.}
We evaluate performance along three dimensions:
(i) \textbf{Image fidelity.} We introduce LPIPS-U, a variant of LPIPS~\citep{zhang2018unreasonable}, to measure perceptual similarity over unedited regions, and employ DINO-ViT~\citep{caron2021emerging} to assess semantic consistency.
(ii) \textbf{Constraint adherence.} Spatial and operational constraints are evaluated using Spatial Accuracy (SA) and Constraint Satisfaction Rate (CSR). GroundingDINO~\citep{liu2023grounding} is used to localize referenced objects and automatically verify spatially constrained operations.
(iii) \textbf{Instruction completion.} For high-level reasoning evaluation, we adopt Qwen3VL~\citep{qwen3} to score Physical Consistency (PC) and Instruction Compliance (IC).
We further report the Multi-step Score (MS) to quantify overall success in multi-action editing scenarios, which is particularly important for complex benchmarks such as \textsc{I2E-Bench}.
Implementation details are provided in Appendix~\ref{app:eval}.

\begin{figure*}[t]
  \centering
  \includegraphics[width=0.96\textwidth]{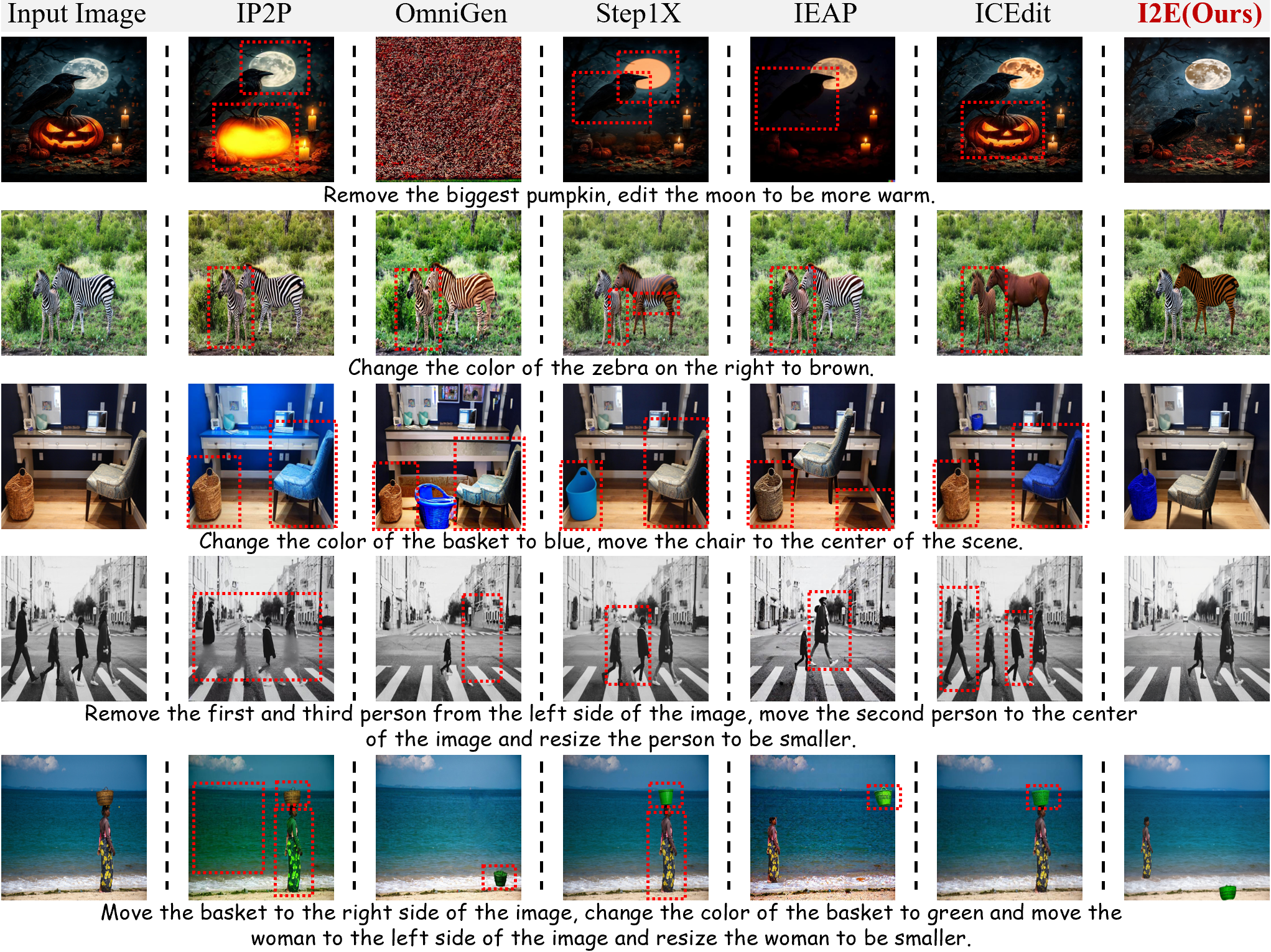}
  \caption{The results of the qualitative comparison on \textsc{I2E-Bench}.}
  \label{fig:qualitative}
  \vspace{-10pt}
\end{figure*}

\begin{figure*}[t]
  \centering
  \includegraphics[width=0.96\textwidth]{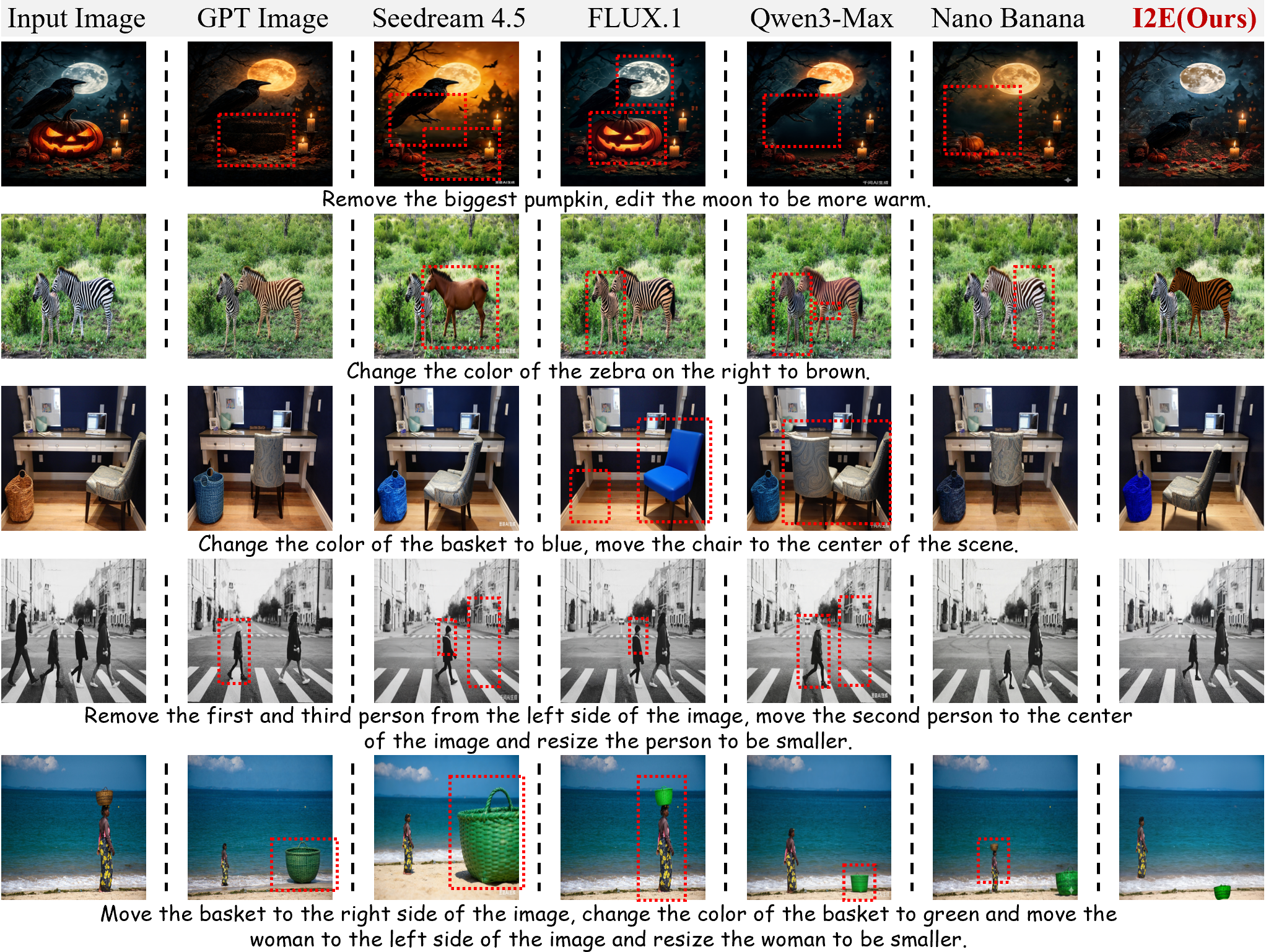}
  \caption{Qualitative results on I2E-Bench compared to commercial models.}
  \label{fig:close}
\end{figure*}

\subsection{Quantitative Results}

\begin{table}[t]
  \centering
  \scriptsize
  \setlength{\tabcolsep}{3.2pt}
  \begin{tabular}{l|ccccccc}
    \toprule
    \textbf{Method} &
    LPIPS-U$\downarrow$ &
    DINO$\uparrow$ &
    SA$\uparrow$ &
    CSR$\uparrow$ &
    PC$\uparrow$ &
    IC$\uparrow$ &
    MS$\uparrow$ \\
    \midrule
    IP2P      & 0.2137 & 0.9565 & 0.4654 & 0.6400 & 0.7270 & 0.3747 & 0.3765 \\
    OmniGen  & 0.2072 & 0.9507 & 0.5263 & 0.6900 & 0.5620 & 0.3813 & 0.4300 \\
    Step1X    & \textbf{0.0353} & \textbf{0.9892} & 0.4861 & 0.6300 & \underline{0.8400} & \underline{0.5898} & 0.5483 \\
    IEAP      & 0.1532 & 0.9651 & \underline{0.6213} & \underline{0.7700} & 0.6290 & 0.5109 & \underline{0.5583} \\
    ICEdit    & 0.0778 & 0.9816 & 0.4825 & 0.6200 & 0.7560 & 0.5125 & 0.5199 \\
    \midrule
    \textbf{I2E} &
    \underline{0.0754} &
    \underline{0.9821} &
    \textbf{0.6923} &
    \textbf{0.8700} &
    \textbf{0.9210} &
    \textbf{0.8645} &
    \textbf{0.8074} \\
    \bottomrule
  \end{tabular}
  \caption{Quantitative comparison on \textsc{I2E-Bench}.
  \textbf{Bold}: best; \underline{underline}: second best.}
  \label{tab:quantitative_results}
  \vspace{-10pt}
\end{table}

\begin{table}[t]
  \centering
  \scriptsize
  \setlength{\tabcolsep}{5pt}
  \begin{tabular}{l|cccccc}
    \toprule
    \textbf{Method} & LPIPS-U$\downarrow$ & DINO$\uparrow$ & SA$\uparrow$ & CSR$\uparrow$ & PC$\uparrow$ & IC$\uparrow$ \\
    \midrule
    \multicolumn{7}{c}{\cellcolor[HTML]{F2F2F2}\textit{MagicBrush}} \\
    \midrule
    IP2P      & 0.0865 & 0.8765 & 0.6840 & 0.9700 & 0.7680 & 0.4808 \\
    OmniGen  & \underline{0.0456} & \underline{0.9254} & 0.6820 & 0.9700 & 0.6600 & 0.6435 \\
    Step1X    & 0.0476 & 0.9226 & \underline{0.6920} & 0.9800 & 0.7490 & \textbf{0.8617} \\
    IEAP      & 0.1893 & 0.7085 & 0.7020 & \underline{0.9900} & 0.7640 & 0.5467 \\
    ICEdit    & 0.0554 & 0.9251 & 0.6820 & 0.9700 & 0.7240 & 0.7020 \\
    \midrule
    \textbf{I2E} & \textbf{0.0446} & \textbf{0.9581} & \textbf{0.7120} & \textbf{1.0000} & \textbf{0.8668} & \underline{0.7924} \\
    \midrule
    \multicolumn{7}{c}{\cellcolor[HTML]{F2F2F2}\textit{EmuEdit}} \\
    \midrule
    IP2P      & 0.1392 & 0.8734 & 0.6539 & 0.9300 & 0.8470 & 0.5676 \\
    OmniGen  & 0.0923 & 0.8625 & 0.6923 & \underline{0.9700} & 0.7340 & 0.6563 \\
    Step1X    & 0.0732 & 0.9029 & \underline{0.6969} & \underline{0.9700} & 0.8190 & \underline{0.7672} \\
    IEAP      & 0.1533 & 0.7592 & 0.6881 & 0.9600 & 0.8230 & 0.6645 \\
    ICEdit    & \textbf{0.0493} & \underline{0.9261} & 0.6660 & 0.9400 & \underline{0.8570} & 0.7238 \\
    \midrule
    \textbf{I2E} & \underline{0.0565} & \textbf{0.9404} & \textbf{0.7183} & \textbf{1.0000} & \textbf{0.8952} & \textbf{0.8107} \\
    \bottomrule
  \end{tabular}
  \caption{Quantitative comparison on MagicBrush and EmuEdit.
  \textbf{Bold}: best; \underline{underline}: second best.}
  \label{tab:combined_results}
  \vspace{-5pt}
\end{table}

Tables~\ref{tab:quantitative_results} and~\ref{tab:combined_results} report quantitative comparisons between our method and state-of-the-art baselines on \textsc{I2E-Bench}, MagicBrush, and EmuEdit.

\paragraph{Results on \textsc{I2E-Bench}}
As shown in Table~\ref{tab:quantitative_results}, our method consistently outperforms all baselines across most metrics.
In particular, I2E achieves a substantial improvement on the Multi-step Score (MS), exceeding the second-best method by nearly $0.25$.
This gain mainly stems from the instance-level disentanglement introduced by the Decomposer, which isolates editing effects and effectively mitigates error accumulation in multi-step interactions.
Moreover, equipped with the VLA Editor for explicit physics-aware reasoning and action decomposition, our method also achieves clear advantages in constraint-related metrics, including CSR, IC, and PC.
These results demonstrate the effectiveness of reformulating image editing as structured interaction within an explicit physical environment.

We note that LPIPS-U and DINO scores are slightly lower than those of Step1X, which is primarily attributable to the currently adopted background restoration module.
As an open framework, I2E can readily incorporate stronger restoration models to further improve perceptual fidelity.

\begin{table}[t]
  \centering
  \scriptsize
\renewcommand{\arraystretch}{0.8}
  \begin{tabular}{l|cc|cc}
    \toprule
    \textbf{Method} 
    & Avg. Score $\uparrow$ & Rank 
    & Mean Rank $\downarrow$ & Rank \\
    \midrule
    IP2P      
      & 2.24 & 6 & 4.76 & 6 \\
    OmniGen  
      & 3.39 & 5 & 4.08 & 5 \\
    Step1X    
      & 4.32 & 3 & \underline{3.37} & 2 \\
    IEAP      
      & \underline{4.40} & 2 & 3.39 & 3 \\
    ICEdit    
      & 3.83 & 4 & 3.60 & 4 \\
    \midrule
    \textbf{I2E (Ours)} 
      & \textbf{7.42} & \textbf{1} & \textbf{1.79} & \textbf{1} \\
    \bottomrule
  \end{tabular}
  \caption{Human evaluation results on \textsc{I2E-Bench}.
  Rank columns indicate the ordering among all compared methods.
  \textbf{Bold} and \underline{underline} denote the best and second-best results, respectively.}
  \label{tab:human_evaluation}
  \vspace{-10pt}
\end{table}

\paragraph{Results on MagicBrush and EmuEdit}
Table~\ref{tab:combined_results} summarizes the results on two widely used public benchmarks.
Our method shows consistent superiority across datasets.
On MagicBrush, I2E achieves the best performance on all reported metrics.
On EmuEdit, although LPIPS-U is marginally lower than ICEdit, I2E yields a significant improvement in instruction compliance (IC), with a relative gain of nearly $9$\%.
Notably, I2E attains perfect constraint satisfaction (CSR = $1.0000$) and the highest physical consistency (PC) on both datasets, highlighting the effectiveness of explicit spatial planning and physical constraint modeling for complex image editing.

\begin{table*}[t]
  \centering
  \scriptsize
  \setlength{\tabcolsep}{2.6pt}
  \begin{tabular}{l|cc|cc|cc|cc|cc|cc|cc}
    \toprule
    \textbf{Variant} &
    \multicolumn{14}{c}{\textbf{Metrics (Value / $\Delta$)}} \\
    \cline{2-15}
     &
    \multicolumn{2}{c|}{LPIPS-U$\downarrow$} &
    \multicolumn{2}{c|}{DINO$\uparrow$} &
    \multicolumn{2}{c|}{SA$\uparrow$} &
    \multicolumn{2}{c|}{CSR$\uparrow$} &
    \multicolumn{2}{c|}{PC$\uparrow$} &
    \multicolumn{2}{c|}{IC$\uparrow$} &
    \multicolumn{2}{c}{MS$\uparrow$} \\
    \midrule

    Full (Ours) &
    0.0754 & -- &
    0.9821 & -- &
    0.6923 & -- &
    0.8700 & -- &
    0.9210 & -- &
    0.8645 & -- &
    0.8074 & -- \\
    \midrule

    \multicolumn{15}{c}{\cellcolor[HTML]{F2F2F2}\textit{Decomposer}} \\

    w/o BGR &
    0.0347 & $-0.0407$ &
    0.9837 & $+0.0016$ &
    0.6562 & $-0.0361$ &
    0.8100 & $-0.0600$ &
    0.6540 & $-0.2670$ &
    0.3387 & $-0.5258$ &
    0.5432 & $-0.2642$ \\
    w/o FGR &
    0.0803 & $+0.0049$ &
    0.9797 & $-0.0024$ &
    0.6507 & $-0.0416$ &
    0.7800 & $-0.0900$ &
    0.7450 & $-0.1760$ &
    0.4943 & $-0.3702$ &
    0.6383 & $-0.1691$ \\
    w/o DAG &
    0.0828 & $+0.0074$ &
    0.9787 & $-0.0034$ &
    0.6253 & $-0.0670$ &
    0.7700 & $-0.1000$ &
    0.8490 & $-0.0720$ &
    0.4792 & $-0.3853$ &
    0.5557 & $-0.2517$ \\
    \midrule

    \multicolumn{15}{c}{\cellcolor[HTML]{F2F2F2}\textit{VLA Editor}} \\
    w/o PR &
    0.0822 & $+0.0068$ &
    0.9774 & $-0.0047$ &
    0.6483 & $-0.0440$ &
    0.8000 & $-0.0700$ &
    0.8360 & $-0.0850$ &
    0.4742 & $-0.3903$ &
    0.5142 & $-0.2932$ \\
    w/o AR &
    0.0896 & $+0.0142$ &
    0.9798 & $-0.0023$ &
    0.6177 & $-0.0746$ &
    0.8500 & $-0.0200$ &
    0.8100 & $-0.1110$ &
    0.3151 & $-0.5494$ &
    0.3449 & $-0.4625$ \\
    \bottomrule
  \end{tabular}
  \caption{Quantitative ablation on \textsc{I2E-Bench}.}
  \label{tab:ablation_study}
\end{table*}

\begin{figure*}[t]
  \centering
  \includegraphics[width=0.98\textwidth]{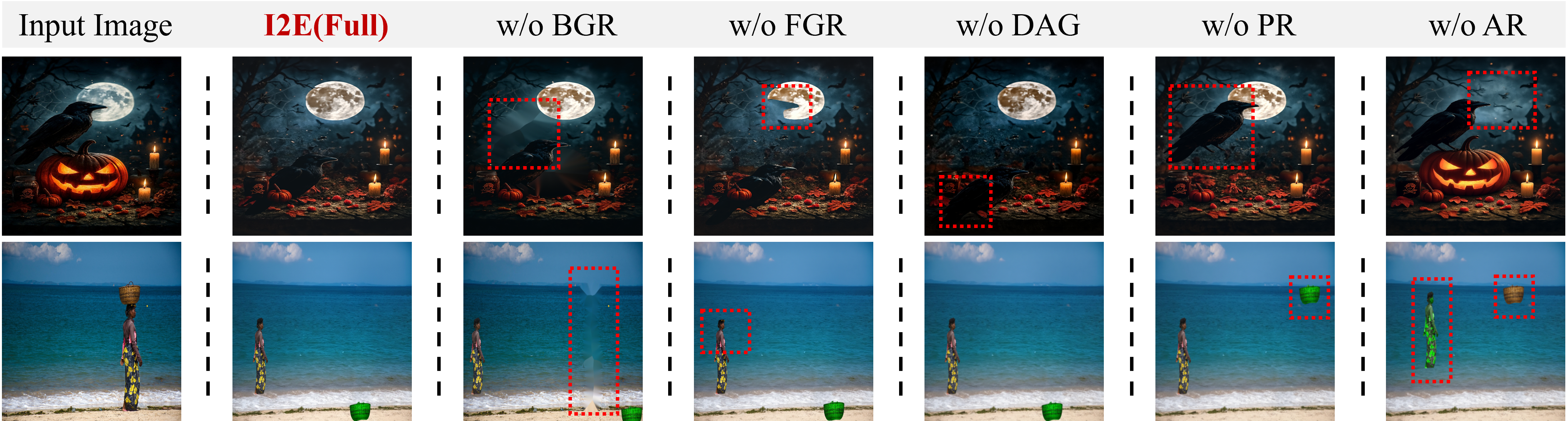}
  \caption{Qualitative ablation on \textsc{I2E-Bench}.}
  \label{fig:ablation}
\end{figure*}

\subsection{Qualitative Results}

Figure~\ref{fig:qualitative} visualizes comparisons on \textsc{I2E-Bench}, highlighting our framework's superiority across three critical dimensions:

\paragraph{Compositionality and Physical Logic.} 
End-to-end models often suffer from semantic coupling in multi-objective tasks. For instance, given the instruction ``\textit{remove the pumpkin and adjust the moon}'' (Row 1), baselines typically miss sub-goals due to attention conflict. I2E, however, accurately disentangles and executes all constraints. Crucially, I2E enforces physical plausibility. Although baselines like Step1X remove the supporting object (pumpkin), it usually surves a clear physical hallucination like leaving the dependent object (crow) floating. Relatively, I2E detects the support loss and triggers a gravity simulation, naturally landing the crow to ensure logical consistency.

\paragraph{Spatial Precision and Attribute Isolation.} 
I2E excels in maintaining geometric integrity during large-scale manipulations. In tasks requiring significant displacement (e.g., ``\textit{move the woman/chair},'' Rows $3$\&$5$), baselines frequently distort object structures or misplace targets. In contrast, our $2.5$D layered representation enables precise translation and scaling. Furthermore, for local attribute editing (e.g., ``\textit{recolor the right zebra},'' Row $2$), our strict layer isolation prevents the attribute leakage (color bleeding to the adjacent zebra) commonly observed in global processing models.

\subsection{Human Evaluation}

We conduct a blind human evaluation to assess human preference for complex image editing results.
Thirty participants each score 10 randomly sampled cases, where six anonymized outputs from I2E and five baselines are shown in random order, using a $0$--$10$ holistic rating. As shown in Table~\ref{tab:human_evaluation}, I2E consistently achieves the highest preference.

\section{Comparison with Commercial Models}
We further present qualitative comparisons between I2E and several recent commercial models, including GPT5.1~\citep{wang2025gptimageedit15mmillionscalegptgeneratedimage}, Seedream4.5~\citep{seedream2025seedream40nextgenerationmultimodal}, FLUX.1-Kontext-dev~\citep{labs2025flux1kontextflowmatching}, Qwen3-Max~\citep{qwen3max}, and Nano Banana~\citep{banana} in Figure~\ref{fig:close}. The results show that I2E outperforms these closed-source models in compositional instruction-based image editing, particularly for tasks requiring precise instance-level control and spatially and physically consistent manipulation.

\section{Ablation Study}
We conduct an ablation study on I2E-Bench to quantify the contribution of each component, as reported in Table~\ref{tab:ablation_study} and Figure~\ref{fig:ablation}. Specifically, we evaluate the effects of foreground reconstruction (FGR), background reconstruction (BGR), DAG-based spatial constraint propagation (DAG), physical reasoning (PR), and action reasoning (AR).

\subsection{Quantitative Analysis.}
Table~\ref{tab:ablation_study} confirms all modules are essential.

\paragraph{Decomposer:} 
Removing \textit{Background Reconstruction} yields artificially better LPIPS-U because the model edits fewer pixels, but causes a sharp drop in IC (to $0.3387$), indicating a failure to handle background-dependent instructions. Removing \textit{Foreground Reconstruction} compromises the geometric integrity of occluded objects, significantly degrading PC and MS. Absence of \textit{DAG-based Spatial Constraint Propagation} lowers CSR ($0.87$ $\to$ $0.77$), confirming the necessity of DAG-based sorting for modeling complex occlusions.
\paragraph{VLA Editor:} 
Without \textit{Action Reasoning}, IC and MS drop to $0.3151$ and $0.3449$, highlighting the critical role of decomposing high-level instructions into atomic actions for long-horizon stability. Disabling \textit{Physical Reasoning} reduces PC ($0.92$ $\to$ $0.84$), verifying that explicit constraints (e.g., gravity and support) are essential to prevent physical anomalies such as floating objects.

\subsection{Qualitative Analysis}
Figure~\ref{fig:ablation} qualitatively supports the above results.
\paragraph{Decomposer:} 
Removing \textit{background reconstruction} causes visible discontinuities, while disabling \textit{foreground reconstruction} leads to incomplete instances after movement or occlusion changes. Without \textit{DAG-based propagation occlusion}, the ordering in crowded scenes will become inconsistent.
\paragraph{VLA Editor:}
Ablating \textit{action reasoning}, edits become incomplete or incorrectly ordered. Disabling \textit{physical reasoning} results in floating objects or invalid support relations.

\section{Conclusion}

We propose I2E, a structured framework for instruction-based image editing via explicit reasoning and execution. By decomposing instructions into spatially and physically grounded actions, I2E enables more controllable and consistent editing. Future work aims to extend I2E to professional domains like interior design, envisioning the VLA Editor evolving from a passive “Instruction Executor” into an active “Intelligent Designer.”

\section*{Limitations}

Despite the encouraging results, the current framework still faces several limitations that warrant further investigation. First, the quality of scene decomposition is highly dependent on the performance of the underlying foundation models. Although SAM~$2$ and Flux provide state-of-the-art results, they may still struggle with extremely complex occlusion patterns or transparent objects such as glass and water, leading to imprecise segmentation or inpainting artifacts that can propagate to subsequent editing stages. Second, at the rendering stage, the system does not yet fully rely on physical simulation to handle complex material properties, such as soft-body deformation or fluid dynamics, nor fine-grained lighting interactions, such as caustics and interreflections. Consequently, edits involving significant lighting changes or object deformation may lack the highest level of photorealism.

\section*{Acknowledgments}
This paper is supported by the National Natural Science Foundation of China (No. 62406161) and sponsored by CCF-Kuaishou Large Model Explorer Fund (NO. CCF-KuaiShou 2025003).

\bibliography{I2E}

\newpage
\newpage
\appendix

\section{Appendix}
\label{sec:appendix}

\begin{figure*}[t]
  \centering
  \includegraphics[width=\textwidth]{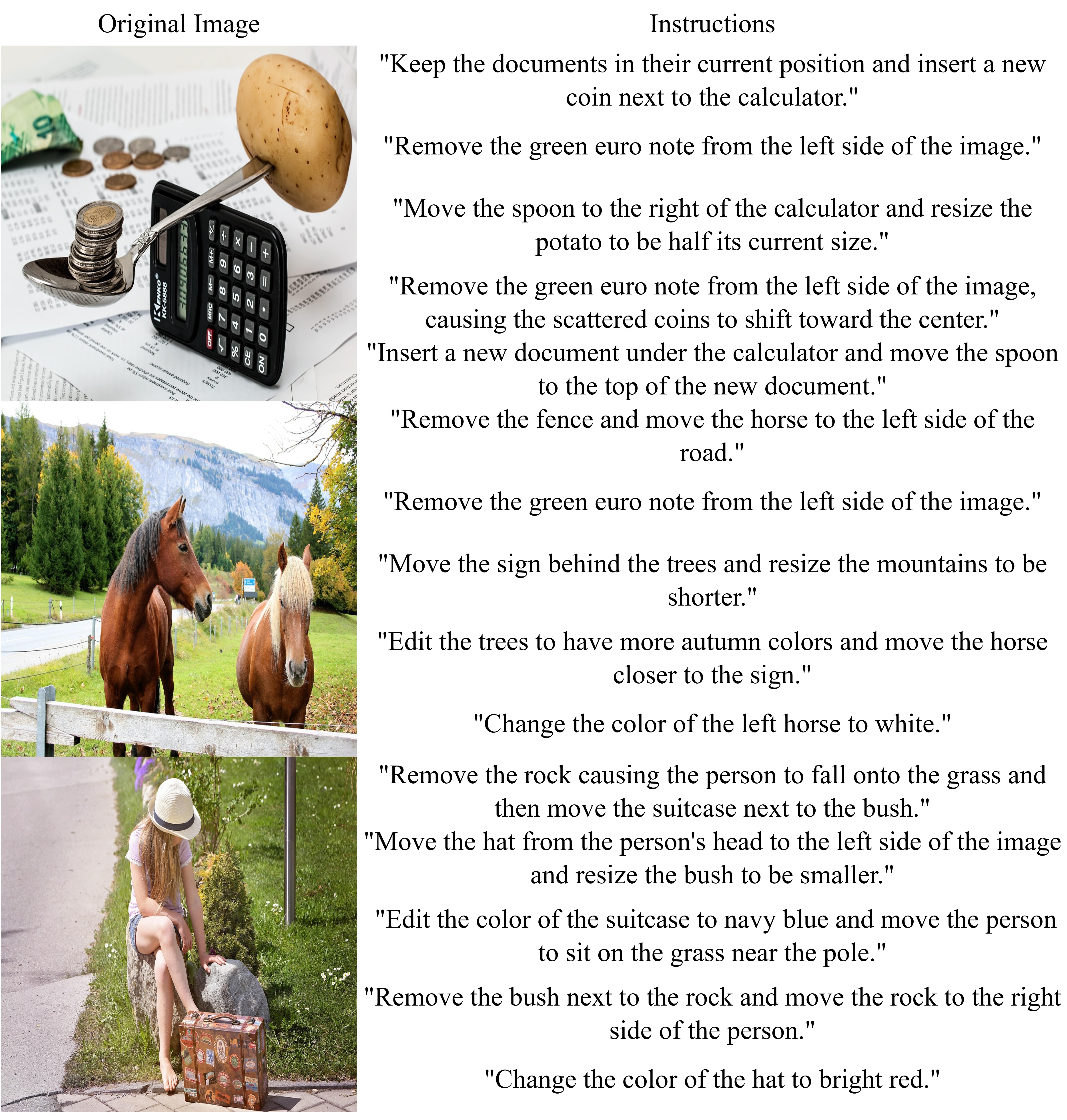}
  \caption{Examples of I2E-bench.}
  \label{fig:dataset}
\end{figure*}
\subsection{Details of I2E-Bench}
\label{app:I2Ebench}
\paragraph{Details of Data Collection}
The 200 base images for I2E-Bench were manually selected from Pixabay\footnote{\url{https://pixabay.com/}}. We intentionally chose this source due to its high-resolution content and the diversity of artistic styles it offers. For each image, we curated a diverse set of $5$--$10$ instructions, ranging from straightforward single-action prompts to intricate multi-step tasks. This design is specifically used to provide a rigorous assessment of the model's capabilities in spatial reasoning and sequential execution. Representative examples are provided in Figure~\ref{fig:dataset}.

\subsection{Theoretical Analysis: Structural Limitations of End-to-End Editing}
\label{app:tho1}
\begin{theorem}[Instruction Collapse under Global Conditioning]
\label{thm:instruction_collapse}
Let an end-to-end image editor condition generation on a single text embedding
$c \in \mathbb{R}^d$ produced from a composite instruction
$T = \{s_1, \dots, s_K\}$, where each $s_k$ denotes a semantically independent sub-instruction.
For sufficiently large $K$, the model cannot guarantee the simultaneous and stable execution
of all sub-instructions in a single forward generation.
\end{theorem}

\begin{proof}[Proof Sketch]
The editor encodes the full instruction as
\begin{equation}
c = E_{\text{text}}(T),
\end{equation}
where $E_{\text{text}}$ has fixed output dimension $d$ independent of $K$.
Each sub-instruction introduces at least one independent semantic factor, so the intrinsic
degrees of freedom required to represent $T$ grow with $K$.
When $K > d$, the mapping $E_{\text{text}}$ is necessarily non-injective, i.e.,
\begin{equation}
\exists\; T \neq T' \quad \text{s.t.} \quad E_{\text{text}}(T) = E_{\text{text}}(T'),
\end{equation}
implying unavoidable information loss.

Even when $K \le d$, the single embedding $c$ jointly represents all sub-instructions.
Since no constraint enforces disentanglement, the induced sub-instruction representations
occupy overlapping directions in $\mathbb{R}^d$:
\begin{equation}
\langle E_{\text{text}}(s_i), E_{\text{text}}(s_j) \rangle \neq 0,
\quad i \neq j.
\end{equation}
During generation, $c$ is broadcast to all spatial or latent tokens via attention.
Without an explicit routing mechanism, sub-instructions compete for the same conditioning
channel, leading to unstable or selective execution.
This phenomenon is referred to as \emph{instruction collapse}.
\end{proof}

\medskip

\begin{proposition}[Elimination of Same-Step Gradient Conflict by Sequential Decomposition]
\label{prop:gradient_conflict}
Consider the joint optimization objective
\begin{equation}
\mathcal L_{\text{joint}}(\theta)
= -\log p_\theta(x_0 \mid c_1, \dots, c_K),
\end{equation}
where $\{c_k\}$ denote sub-instruction conditions.
Let $\mathbf g_k$ be the gradient contribution induced by $c_k$.
Then the gradient of $\mathcal L_{\text{joint}}$ contains cross terms
$\langle \mathbf g_i, \mathbf g_j\rangle$ for $i \neq j$, which may be negative.

By contrast, if the objective is decomposed into a sequence of conditional sub-objectives
\begin{equation}
\mathcal L_{\text{seq}}(\theta)
= \sum_{k=1}^K
\mathcal L_k(\theta),
\end{equation}
\begin{equation}
\mathcal L_k
= -\log p_\theta(x_0^{(k)} \mid x_0^{(k-1)}, c_k),
\end{equation}
then no cross-gradient terms arise within the same optimization step.
\end{proposition}

\begin{proof}[Proof Sketch]
For the joint objective,
\begin{equation}
\nabla_\theta \mathcal L_{\text{joint}}
= \sum_{k=1}^K \mathbf g_k,
\end{equation}
and the squared gradient norm expands as
\begin{equation}
\|\nabla_\theta \mathcal L_{\text{joint}}\|^2
= \sum_k \|\mathbf g_k\|^2
+ \sum_{i\neq j} \langle \mathbf g_i, \mathbf g_j\rangle,
\end{equation}
where negative inner products correspond to gradient conflict.

Under the sequential formulation, each sub-objective $\mathcal L_k$ depends on a single
condition $c_k$.
At optimization step $k$, the gradient is
\begin{equation}
\nabla_\theta \mathcal L_k = \mathbf g_k,
\end{equation}
and no terms involving $\mathbf g_j$ for $j \neq k$ appear in the same update.
Thus, while interactions across steps may still exist, the same-step gradient conflict
inherent to joint optimization is structurally eliminated.
\end{proof}

\subsection{Theoretical Analysis: Global Degradation in Multi-Round Editing}
\label{app:tho2}
\begin{observation}[Global Degradation under Iterative End-to-End Editing]
\label{obs:global_degradation}
Consider an end-to-end image editor that applies a sequence of sub-instructions
$\{c_1, \dots, c_N\}$ through repeated generative updates
\begin{equation}
x^{(t+1)} = F_\theta(x^{(t)}, c_t).
\end{equation}
Even when each sub-instruction $c_t$ is intended to modify only a localized region,
the deviation of non-target regions increases with the number of editing rounds $N$.
\end{observation}

\begin{proof}[Explanation]
Each update $F_\theta$ operates on the full image or latent representation.
Due to finite latent capacity and stochastic sampling,
the reconstruction at each step can be written as
\begin{equation}
x^{(t+1)} = x^{(t)} + \epsilon^{(t)},
\end{equation}
where $\epsilon^{(t)}$ denotes the reconstruction error affecting both target
and non-target regions.

Since updates are applied recursively, the accumulated deviation after $N$ rounds is
\begin{equation}
x^{(N)} - x^{(0)} = \sum_{t=1}^N \epsilon^{(t)} + \mathcal{O}(\epsilon^2),
\end{equation}
which implies
\begin{equation}
\mathbb{E}\|x^{(N)} - x^{(0)}\|
\ge \sum_{t=1}^N \mathbb{E}\|\epsilon^{(t)}\|.
\end{equation}

Moreover, modern generative editors rely on dense token mixing mechanisms,
causing local modifications to perturb global feature statistics.
As a result, errors introduced in early rounds propagate to non-target regions
and are amplified over subsequent rounds.

This observation explains the empirically observed background degradation
and increasing pixel-wise deviation outside the edited regions
as the number of editing rounds grows.
\end{proof}

\subsection{Details of Implementation}
\label{app:implementation}
\paragraph{Decomposer}
Our perception pipeline is implemented using Qwen-VL~\cite{qwen3} as the MLLM, Grounding DINO~\cite{liu2023grounding} for bounding box localization, and SAM~\cite{kirillov2023segment} for pixel-level segmentation. For layer completion, we employ Flux-Fill~\cite{labs2025flux1kontextflowmatching} to perform generative outpainting and content synthesis. Specifically, we use the MLLM to generate detailed descriptive prompts for each occluded instance; these prompts guide Flux-Fill to hallucinate the missing textures and structures behind occlusions, resulting in the final RGBA layers $\{\tilde{I}_i\}$. Simultaneously, to ensure a clean slate for downstream editing, we utilize OmniEraser~\cite{wei2025omnieraserremoveobjectseffects} to perform background inpainting. It effectively removes the foreground instances by filling the holes in the original image using surrounding context, thereby yielding a spatially coherent and artifact-free background $B$.

\paragraph{Physical Reasoning Prompt}
\label{app:physical}
We utilize a multimodal large language model (MLLM) to perform explicit physical reasoning over the image environment.
The model is instructed to simulate physical constraints such as gravity, support, and balance using a structured
``mind's eye'' reasoning process.
Below we provide the full prompt used for physical reasoning in our framework.
\begin{tcolorbox}[
    colback=gray!5,
    colframe=gray!50,
    boxrule=0.5pt,
    left=6pt,
    right=6pt,
    top=6pt,
    bottom=6pt
]
\ttfamily
\small
\begin{spacing}{1.05}

You are a physics reasoning expert. Analyze this scene and the user instruction.

Scene Description:
\{scene\_desc\}

User Instruction: "\{text\_prompt\}"

=== Physical World Simulation Rules (Mind's Eye) ===

Apply the following explicit physics rules:

**Gravity Rules:**
- Rule 1: If Support Object X is removed, Supported Object Y MUST fall
- Rule 2: All objects are affected by gravity unless supported

**Support Rules:**
- Rule 3: Object Y is supported by X if and only if Y is above X and in contact
- Rule 4: If X supports Y, removing X causes Y to lose support

**Balance Rules:**
- Rule 5: An object remains balanced if its center of mass is within the support base
- Rule 6: If support shifts causing center of mass to move outside base, object loses balance

=== Reasoning Steps ===

1. Identify Target Objects
2. Current State
3. Action Effects
4. Physical Constraints
5. Secondary Effects
6. Final State

=== Output Format ===

Provide structured JSON:
\{
    "reasoning": "...",
    "target\_objects": [...],
    "constraints": [...],
    "affected\_objects": [...],
    "predicted\_outcome": "..."
\}

\end{spacing}
\end{tcolorbox}

\paragraph{Atomic Actions planner prompt}
\label{app:planner}
We utilize a multimodal large language model (MLLM) to perform explicit chain-of-thought (CoT) reasoning over the image environment.
The MLLM acts as an action planner that translates the reasoning results into a sequence of executable atomic actions,
which are subsequently applied to the scene to achieve the desired editing goal.
Below we provide the full prompt used for the action planner in our framework.

\begin{tcolorbox}[
    colback=gray!5,
    colframe=gray!50,
    boxrule=0.5pt,
    left=6pt,
    right=6pt,
    top=6pt,
    bottom=6pt
]
\ttfamily
\small
\begin{spacing}{1.05}

You are an action planning expert. Based on the physical reasoning, generate atomic actions.

Scene Description:
\{scene\_desc\}

User Instruction: "\{text\_prompt\}"

Physical Reasoning Result:
\{reasoning\_result\}

Generate a sequence of atomic actions to achieve the instruction while respecting physical constraints.

=== Image Coordinate System ===

Image Dimensions: \{width\} $\times$ \{height\} pixels

Reference Points:
- Top-Left: (0, 0)
- Top-Right: (\{width\}, 0)
- Bottom-Left: (0, \{height\})
- Bottom-Right: (\{width\}, \{height\})
- Image Center: (\{width/2\}, \{height/2\})

Coordinate Rules:
- Origin at top-left
- X increases left to right
- Y increases top to bottom
- All coordinates must satisfy image bounds

=== Available Atomic Actions ===

REMOVE(object\_id) \\
MOVE(object\_id, x, y) \\
KEEP(object\_id) \\
FALL(object\_id, delta\_y) \\
RESIZE(object\_id, scale) \\
RETOUCH(object\_id, brightness, contrast, color, sharpness) \\
EDIT(object\_id, prompt, edit\_type) \\
INSERT(prompt, x, y, width, height, layer\_relation)

Guidelines:
1. Respect gravity and support relations
2. Maintain physical plausibility
3. Execute actions in correct order
4. Strictly adhere to image boundaries

=== Output Format ===

\{
  "action\_sequence": [
    \{ "action": "...", "object\_id": ..., "reason": "..." \}
  ]
\}

Generate the action sequence.

\end{spacing}
\end{tcolorbox}

\subsection{Details of Evaluation Metrics}
\label{app:eval}

\paragraph{Comprehensive Metrics.}
Beyond fundamental metrics like redesigned LPIPS-U and DINO-ViT for fidelity and perceptual quality, we specifically design SA and CSR to quantify constraint-following capabilities. To further capture the nuances of instruction completion, we incorporate MLLM-based metrics, including PC, IC, and MS. The specifics of these metrics are detailed in the subsequent discussion.

\paragraph{LPIPS-U (Unedited-Region LPIPS).}
A key challenge in image editing evaluation is disentangling edit correctness from background preservation.
To emphasize preservation, we define \textbf{LPIPS-U} that measures perceptual distance \emph{only on unedited regions}.

Let $I$ be the original image and $\hat{I}$ be the edited image.
We first estimate a binary edit mask $M \in \{0,1\}^{H \times W}$, where $M(p)=1$ indicates edited pixels and $M(p)=0$ indicates unedited pixels; $\bar{M}=1-M$ denotes the unedited mask.
Let $\phi_\ell(\cdot)$ be the feature map at layer $\ell$ of a fixed perceptual backbone (as in LPIPS), and $\odot$ denote element-wise masking (broadcast to channels).
LPIPS-U is defined as
\begin{equation}
\begin{aligned}
\mathrm{LPIPS\text{-}U}(I,\hat{I}) = & \sum_{\ell \in \mathcal{L}} w_\ell \cdot \\
& \left\| \left(\phi_\ell(I) - \phi_\ell(\hat{I})\right) \odot \bar{M}_\ell \right\|_2,
\end{aligned}
\end{equation}
where $\bar{M}_\ell$ is the unedited mask resized to match the spatial resolution of $\phi_\ell(\cdot)$, and $\{w_\ell\}$ are layer weights.

Unlike LPIPS, LPIPS-U suppresses the contribution from edited regions, making it more sensitive to \emph{unintended changes} outside the target area (i.e., background/irrelevant-region degradation), which is crucial for instruction-based editing.

\paragraph{Constraint following Metrics (SA and CSR).}
We quantify spatial adherence using a GroundingDINO-based detector by comparing object presence and location between $I$ and $\hat{I}$.

For each sample $i$, suppose the prompt yields a set of spatial constraints $\mathcal{C}_i$.
Each constraint $c \in \mathcal{C}_i$ specifies a target object and an expected spatial relation (e.g., left/right/top/bottom/center) and implies an operation type (e.g., \textsc{Remove}, \textsc{Move}, \textsc{Insert}).
Let $a_{i,c} \in [0,1]$ be the per-constraint accuracy computed from detections (object existence and/or normalized center coordinates).

We define \textbf{Spatial Accuracy (SA)} as the mean accuracy across all evaluated constraints:
\begin{equation}
\mathrm{SA}
=
\frac{1}{\sum_i |\mathcal{C}_i|}
\sum_i \sum_{c \in \mathcal{C}_i} a_{i,c}.
\end{equation}

We further define \textbf{Constraint Satisfaction Rate (CSR)} as the fraction of constraints that are satisfied under a threshold $\tau$ (e.g., $\tau=0.7$):
\begin{equation}
\mathrm{CSR}
=
\frac{1}{\sum_i |\mathcal{C}_i|}
\sum_i \sum_{c \in \mathcal{C}_i}
\mathbb{I}\!\left[a_{i,c} \ge \tau\right],
\end{equation}
where $\mathbb{I}[\cdot]$ is the indicator function.
Intuitively, SA reflects average spatial precision, while CSR measures strict compliance frequency.

\paragraph{Instruction-completion evaluation Metrics (PC, IC, and MS).}
We utilize a multimodal learning model (Qwen3-VL) to generate structured assessments that better align with human perception and instruction execution. Scores range from 1 to 10, ultimately mapping to a 0-1 scale.

\textbf{Physical Consistency (PC).}
The judge enumerates physical flaws introduced by the edit (e.g., inconsistent shadows/lighting, implausible occlusion, perspective violations, visible artifacts) with severities.
We compute a deduction-style score:
\begin{equation}
\mathrm{PCPC}
=
\mathrm{clip}_{[1,10]}\!\left(
10 - \sum_{k \in \mathcal{E}_{\text{phys}}} d_k
\right),
\end{equation}
where $\mathcal{E}_{\text{phys}}$ is the set of detected physical issues and $d_k$ is the penalty associated with issue $k$ (larger for more severe flaws).

\textbf{Instruction Following (IC).}
The judge decomposes the prompt into a set of atomic, visually verifiable requested edits $\mathcal{R}$, and assigns each request $r\in\mathcal{R}$ a fulfillment score $f_r \in [0,1]$ (fulfilled/partial/not fulfilled).
Additionally, it reports penalties for unrequested changes and non-target-region damage.
We compute:
\begin{equation}
\begin{aligned}
\mathrm{IC} = \mathrm{clip}_{[1,10]} \Biggl( & 10 \cdot \frac{1}{|\mathcal{R}|} \sum_{r \in \mathcal{R}} f_r \\
& - \sum_{u \in \mathcal{E}_{\text{unwanted}}} d_u - \sum_{p \in \mathcal{E}_{\text{preserve}}} d_p \Biggr).
\end{aligned}
\end{equation}
This formulation emphasizes both completing requested edits and avoiding collateral changes.

\textbf{Multi-step Score (MS).}
While IC focuses on localized fidelity, the Multi-step Score (MS) evaluates the model's ability to execute complex, sequential instructions. For multi-step instructions, the judge first lists the expected action steps $\mathcal{S}$ (restricted to explicit edit actions such as \textsc{Remove}, \textsc{Move}, \textsc{Edit}, \textsc{Insert}, \textsc{Resize}), then scores each step $s\in\mathcal{S}$ with success $g_s \in [0,1]$ based on comparing $(I,\hat{I})$.
We define:
\begin{equation}
\mathrm{MS}
=
10 \cdot \frac{1}{|\mathcal{S}|}\sum_{s \in \mathcal{S}} g_s.
\end{equation}
This metric directly measures step-wise executability and is conservative when steps are missing or not clearly satisfied.

\subsection{Details of Evaluation}
Evaluation is conducted on three benchmarks: I2E-Bench, MagicBrush, and EmuEdit. For each benchmark, we randomly sample $100$ image--instruction pairs from the official evaluation sets. All quantitative metrics are computed using the same evaluation protocol and fixed random seeds to ensure consistency and reproducibility. No additional data filtering or manual selection is performed. The same sampled instances are used consistently across all compared methods and evaluation protocols.

\subsection{Details of VLA Editing }
We provides more visualization of our VLA process in Figure~\ref{fig:app_vla}.

\begin{figure*}[t]
  \centering
  \includegraphics[width=\textwidth]{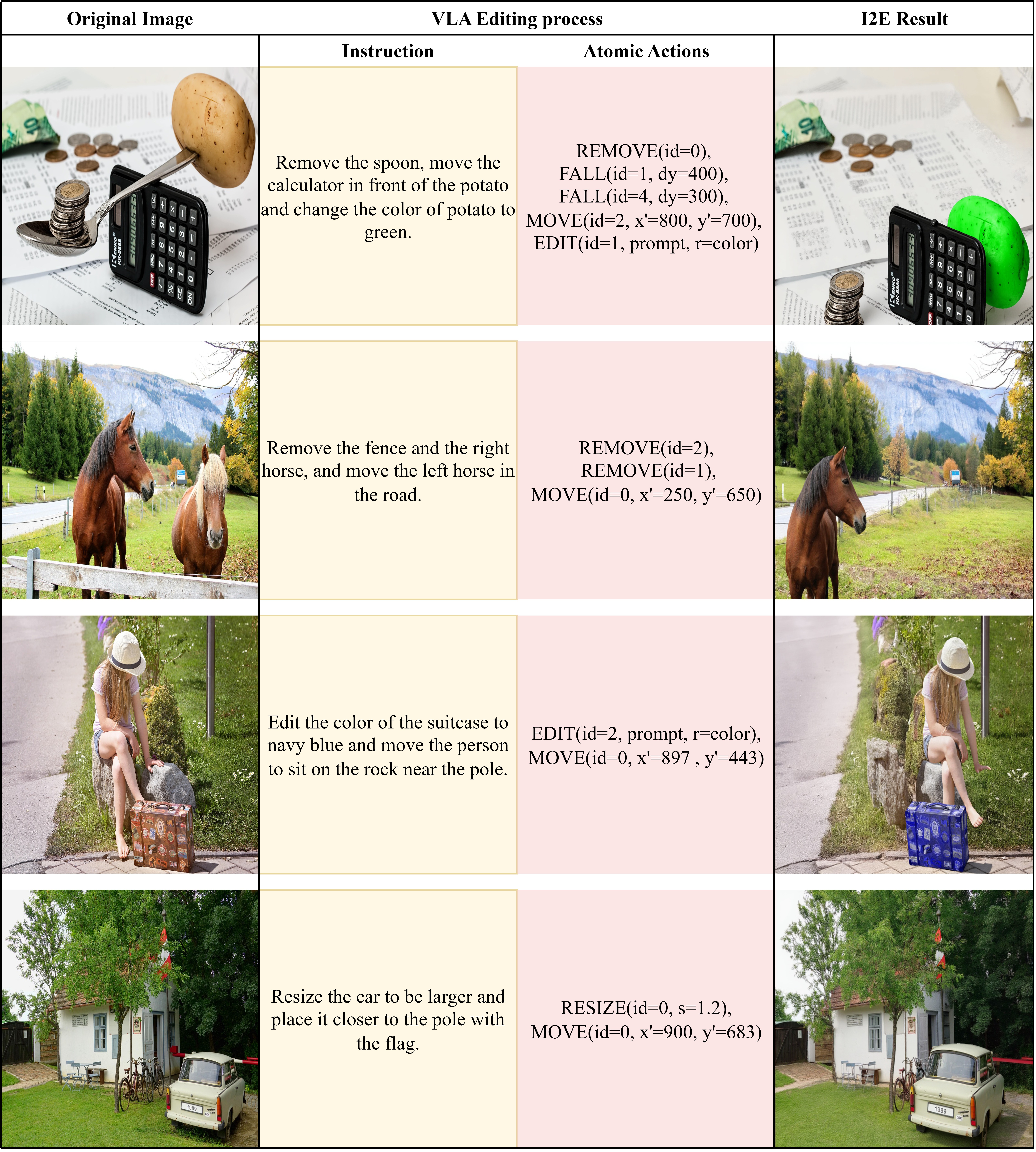}
  \caption{Visualization of VLA process on our proposed I2E-Bench.}
  \label{fig:app_vla}
\end{figure*}

\subsection{Details of Human Evaluation}
Additional visualizations of the human evaluation process are provided in Figure~\ref{fig:he2} and Figure~\ref{fig:he3}. All evaluations were conducted by unpaid volunteers. The data were collected in a blind manner, where participants were not informed of the methods associated with the evaluated results. The evaluation interface and instructions were presented in English. No personal identifying information was collected, and no assumptions or guarantees are made regarding the demographic attributes (e.g., region or gender) of the participants.

\begin{figure*}[t]
  \centering
  \includegraphics[width=\textwidth]{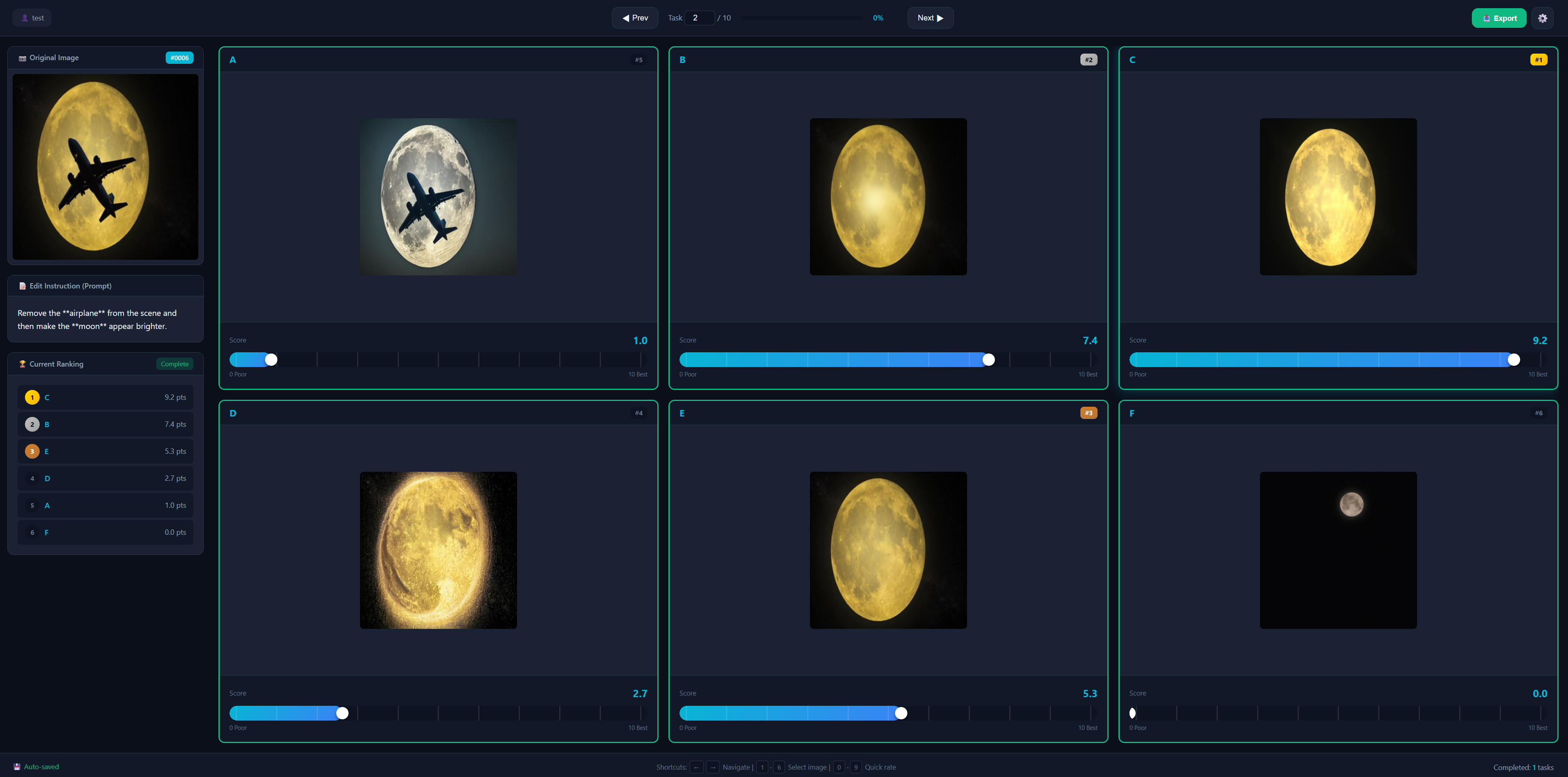}
  \caption{Example of human evaluation system.}
  \label{fig:he2}
\end{figure*}
\begin{figure*}[t]
  \centering
  \includegraphics[width=\textwidth]{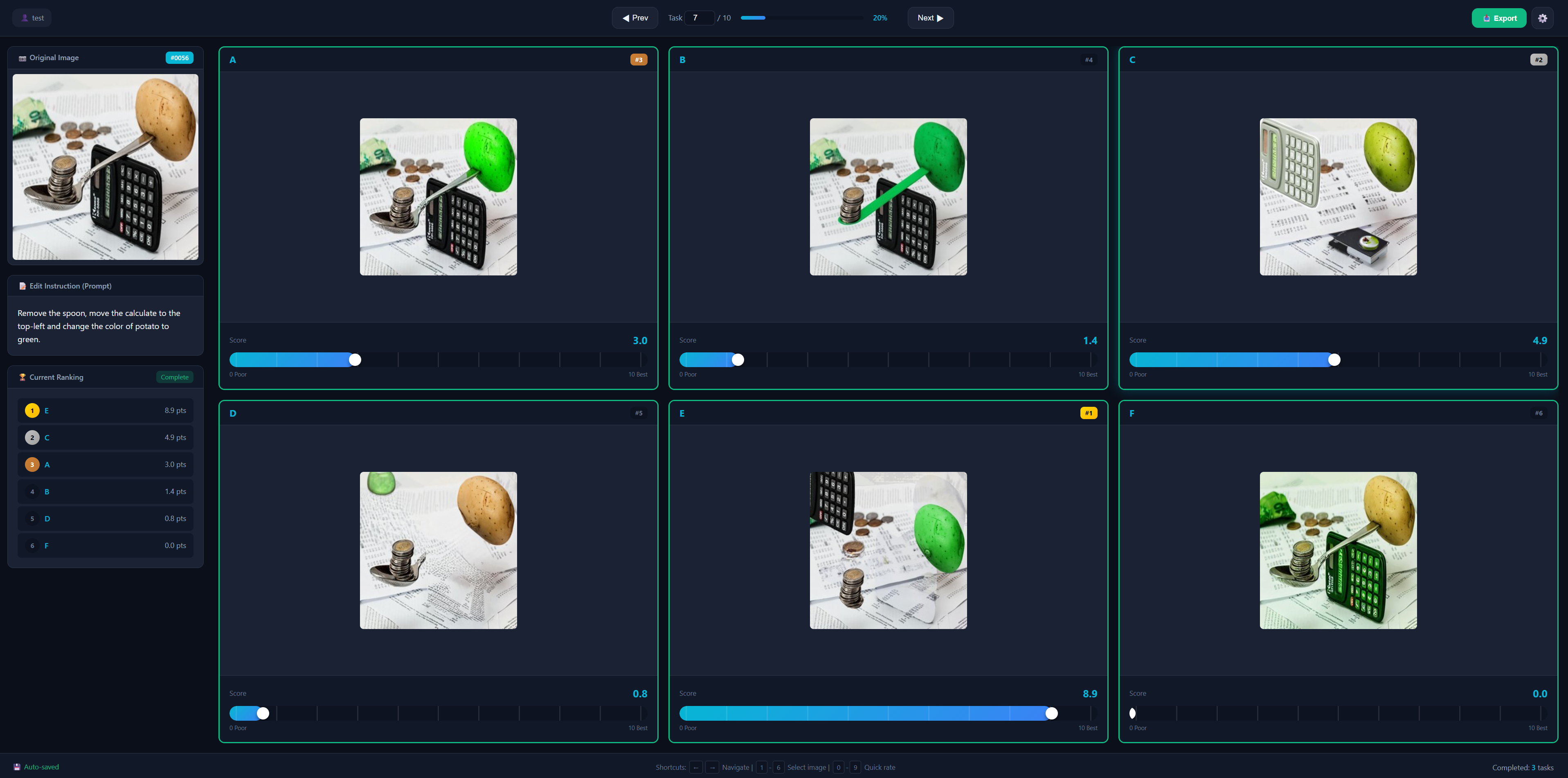}
  \caption{Example of human evaluation system.}
  \label{fig:he3}
\end{figure*}

\subsection{More qualitative Results}
\label{app:more_qualitative_results}
In this section, we present additional qualitative results to further demonstrate the effectiveness of our method across various scenarios in the I2E-Bench. Figures~\ref{fig:app_qua1}, \ref{fig:app_qua2}, and \ref{fig:app_qua3} provide extensive cases demonstrating that our I2E framework consistently outperforms representative baselines in executing multi-step composite instructions, maintaining physical consistency, and ensuring precise spatial localization.

\begin{figure*}[t]
  \centering
  \includegraphics[width=\textwidth]{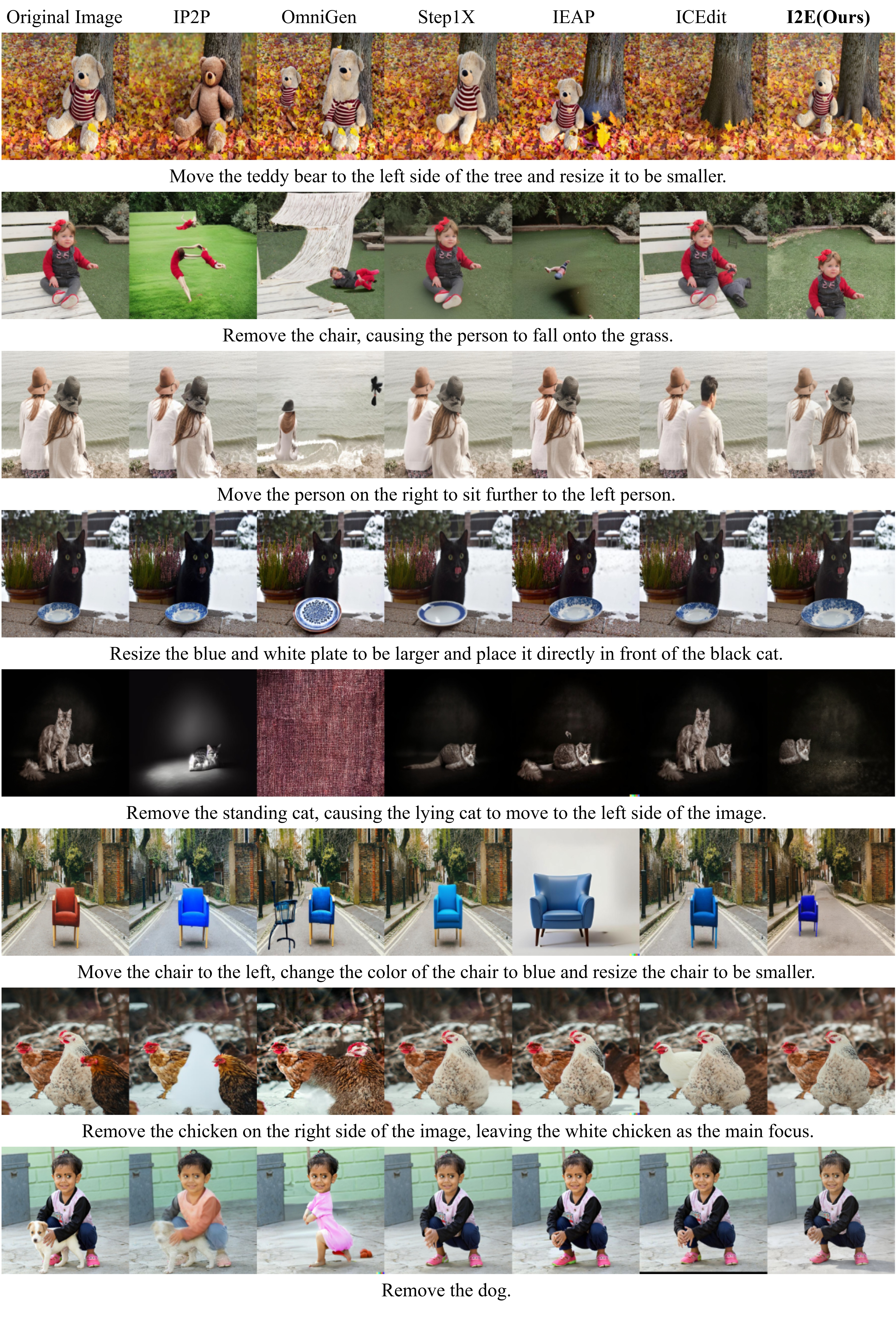}
  \caption{More qualitative results on our proposed I2E-Bench.}
  \label{fig:app_qua1}
\end{figure*}
\begin{figure*}[t]
  \centering
  \includegraphics[width=\textwidth]{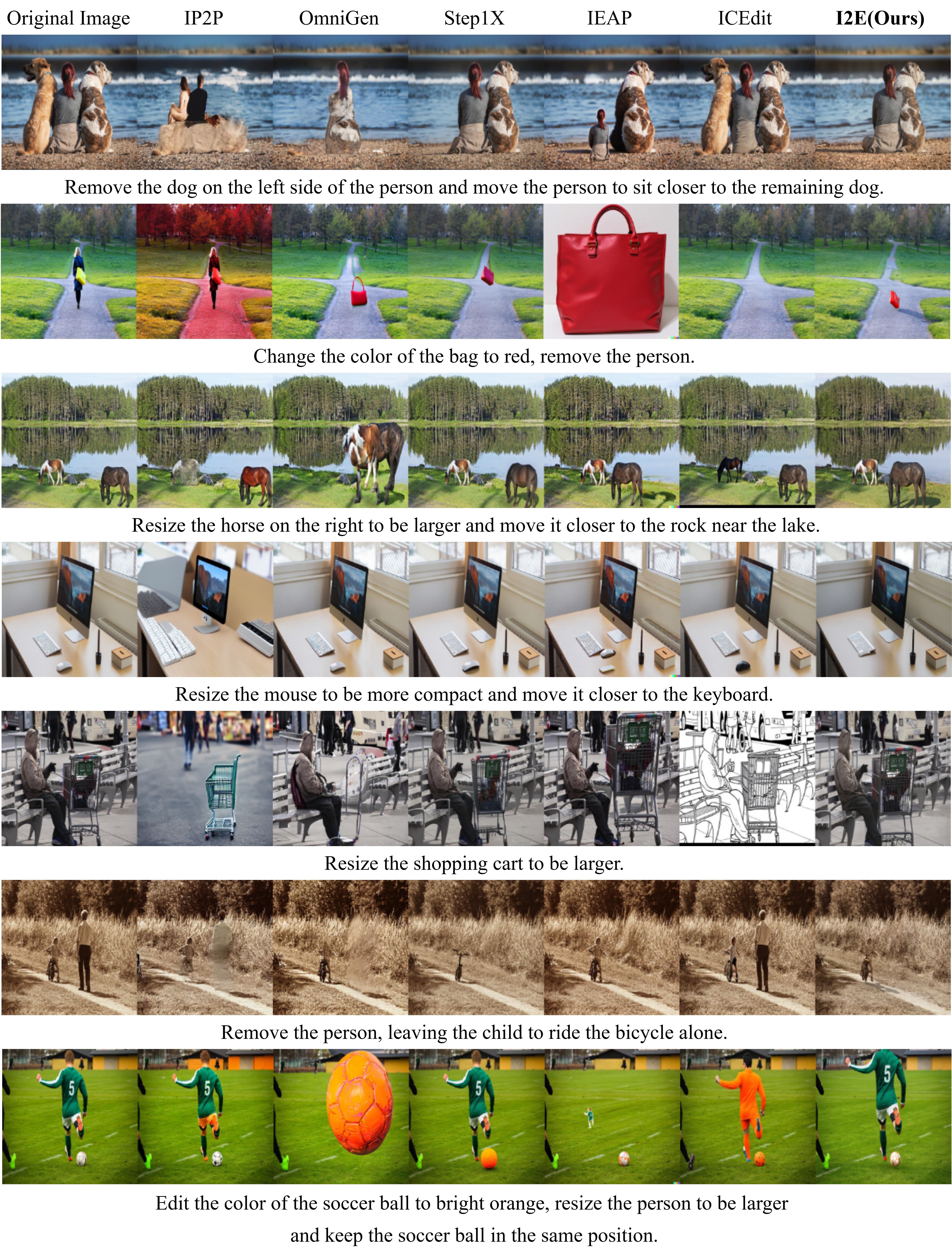}
  \caption{More qualitative results on our proposed I2E-Bench.}
  \label{fig:app_qua2}
\end{figure*}
\begin{figure*}[t]
  \centering
  \includegraphics[width=\textwidth]{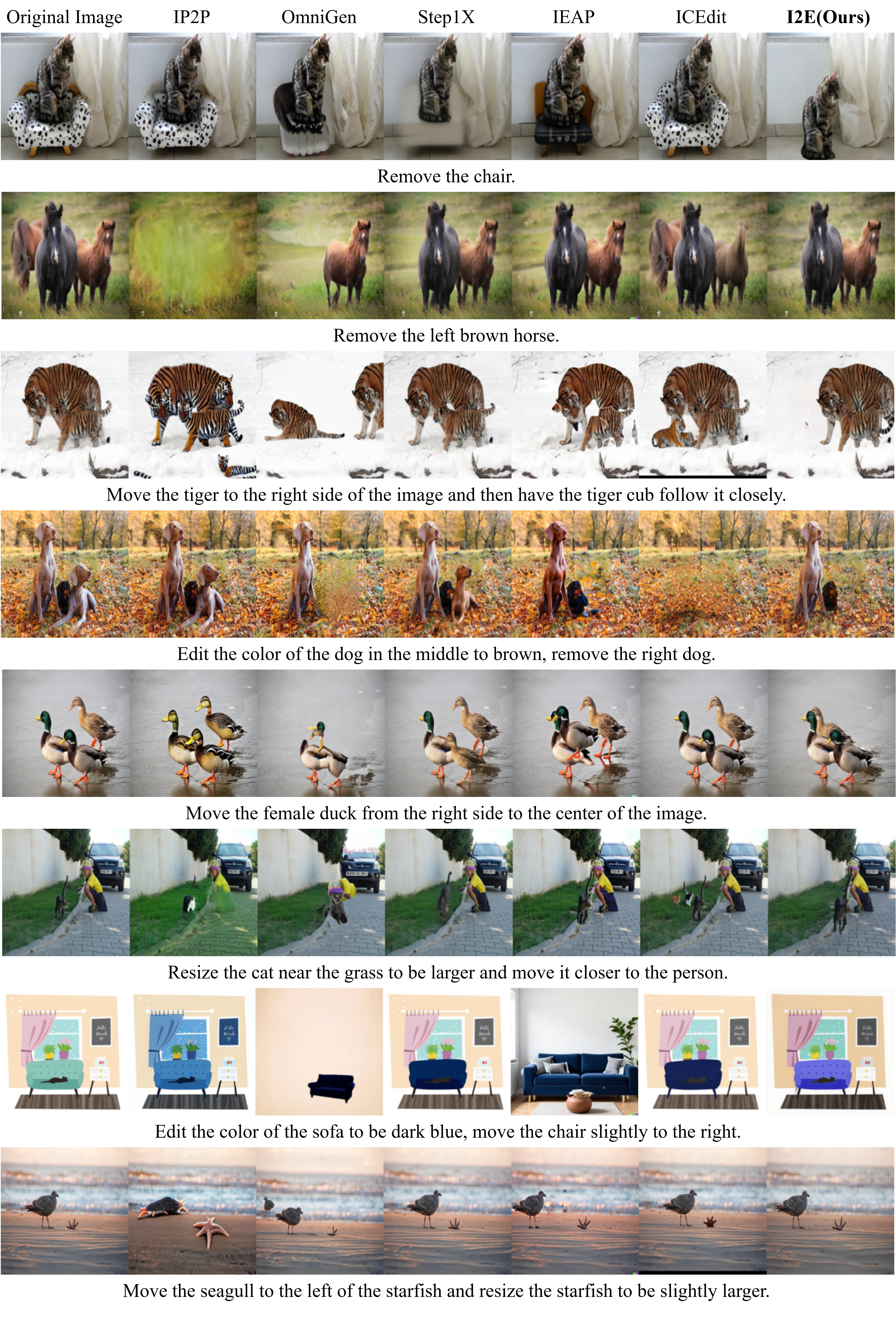}
  \caption{More qualitative results on our proposed I2E-Bench.}
  \label{fig:app_qua3}
\end{figure*}


\subsection{Clarification of Concerns}

\paragraph{Data Contains Personally Identifying Info Or Offensive Content}
All datasets used in this work are collected from publicly available, open-source websites and benchmarks that are explicitly released for non-commercial academic research. We do not intentionally collect or curate any data containing personally identifying information (PII). The data sources are commonly used in prior research and do not target specific individuals. No additional annotation or manual data collection involving human subjects was conducted in this work. As such, the risk of containing sensitive or identifying personal information is minimal.

\paragraph{The Use of Large Language Models (LLMs)}
Large language models (LLMs) were used solely for auxiliary purposes, including grammar checking, formatting refinement, and translation assistance during manuscript preparation. In addition, minor illustrative elements in Figure~1 were generated with the assistance of an LLM-based image generation tool for visualization purposes only (e.g., depicting a human figure). These generated elements are purely illustrative and do not contribute to the core methodology, experimental results, or scientific claims of this work. The overall conceptual design and technical content are entirely human-authored. The use of LLMs does not affect the validity or reproducibility of the experimental findings.

\paragraph{Discuss The License For Artifacts}
All external artifacts used in this work, including datasets and pretrained models, are obtained from open-source research projects and publicly accessible repositories. These artifacts are used in accordance with their original licenses and terms of use, which permit non-commercial academic research. No proprietary or restricted data or models are used.

\paragraph{Artifact Use Consistent With Intended Use}
The usage of existing artifacts in this work is consistent with their intended research purposes as specified by their original authors. All experiments are conducted strictly within a non-commercial, academic research context. Any derived artifacts or outputs produced by our pipeline are intended solely for research and evaluation purposes and are not deployed in real-world or commercial applications.

\paragraph{Documentation Of Artifacts}
The datasets and artifacts used in this work are standard benchmarks in the research community and are documented in their original releases, including information about domains and language usage. Our work focuses on instruction-based multimodal inputs, with all textual instructions provided in English. No claims are made regarding demographic representativeness beyond what is specified in the original datasets.

\paragraph{Model Size And Budget}
This work adopts a training-free pipeline and does not involve training or fine-tuning large-scale models. Therefore, there is no additional model parameter reporting or large-scale computational budget associated with model training. The experiments are conducted using standard academic computing resources for inference and evaluation.

\paragraph{Potential Risks}
This work is based entirely on publicly available, open-source models, data and datasets that are widely used in prior academic research.

Potential risks are limited to the general misuse of image editing technologies, such as generating misleading or manipulated visual content. However, these risks are not unique to the proposed method and are shared by existing generative and editing models. The intended use of this work is strictly non-commercial academic research, and no deployment in real-world decision-making systems is considered. We do not anticipate any significant ethical, societal, or safety concerns arising from the use of this work under its intended scope.

\end{document}